%% file: main.tex
\documentclass{article} % For LaTeX2e
\usepackage{iclr2026_conference,times}

\input{math_commands.tex}

\usepackage{hyperref}
\usepackage{url}
\usepackage{url}
\usepackage{nicefrac}
\usepackage{booktabs}
\usepackage{amsfonts}
\usepackage{bold-extra}
\usepackage{fontawesome5}
\usepackage{multirow}
\usepackage{multicol}
\usepackage{xspace}
\usepackage{xcolor}
\usepackage{makecell}
\usepackage{float}
\usepackage{array}
\usepackage{url}
\usepackage{amsmath}
\usepackage{tabularx}
\usepackage{dcolumn,caption}
\usepackage{arydshln}
\usepackage{tikz}
\usepackage{pgfplots}
\usepackage{framed}
\usepackage{pifont}
\usepackage{bbm}
\usepackage{enumitem}
\usepackage{subcaption}
\usepackage{arydshln}
\usepackage{amssymb}
\usepackage{adjustbox}
\usepackage{pifont}
\usepackage{mdframed}
\usepackage{subcaption}
\usepackage{wrapfig}
\usepackage{graphicx}
\usepackage{subcaption}
\usepackage{booktabs}
\usepackage{xcolor}
\usepackage[most,breakable]{tcolorbox}
\usepackage{color}
\usepackage{etoc}
\usepackage{makecell}
\usepackage{url}
\usepackage{nicefrac}
\usepackage{booktabs}
\usepackage{amsfonts}
\usepackage{bold-extra}
\usepackage{multirow}
\usepackage{multicol}
\usepackage{xspace}
\usepackage{xcolor}
\usepackage{makecell}
\usepackage{titlesec}
\usepackage{graphicx}
\usepackage{hyperref}
\usepackage{fancyvrb}
\usepackage{listings}
\usepackage{cancel}
\usepackage{wrapfig}
\usepackage{setspace}
\usepackage{xcolor}

\usepackage{amsthm}
\newtheorem{lemma}{Lemma}
\newtheorem{theorem}{Theorem}
\newtheorem{assumption}{Assumption}

\title{The Choice of Divergence: A Neglected Key to Mitigating Diversity Collapse in Reinforcement Learning with Verifiable Reward}

\author{
  \centering 
  \parbox{0.95\textwidth}{ % 略微放宽盒子，给单位留出呼吸空间
    \centering
    \begin{tabular}{c}
      \textbf{Long Li}$^{1,2,\dagger}$, \textbf{Zhijian Zhou}$^{1,5,\dagger}$, \textbf{Jiaran Hao}$^{2}$, \textbf{Jason Klein Liu}$^{2}$, \textbf{Yanting Miao}$^{4}$, \textbf{Wei Pang}$^{4}$,\\ 
      \textbf{Xiaoyu Tan}$^{2}$, \textbf{Wei Chu}$^{2}$, \textbf{Zhe Wang}$^{3}$, \textbf{Shirui Pan}$^{3}$, \textbf{Chao Qu}$^{1,6, \ast}$\thanks{Corresponding author. $^\dagger$ Equal contribution.} , \textbf{Yuan Qi}$^{1,6}$ \\[0.6em] 
       $^1$Fudan University \qquad $^2$INFLY TECH \qquad $^3$Griffith University \\
       $^4$University of Waterloo \qquad $^5$Shanghai Innovation Institute \\
       $^6$Shanghai Academy of Artificial Intelligence for Science \\[0.4em] 
      \texttt{seamoke111@gmail.com}
    \end{tabular}
  }
}
\iclrfinalcopy % Uncomment for camera-ready version, but NOT for submission.
\begin{document}

\maketitle

\input{main_text/abstract}

\input{main_text/introduction}

\input{main_text/related_work}

\input{main_text/Preliminaries}

\input{main_text/method}

\input{main_text/experiments}

\input{main_text/analysis}

\input{main_text/conclusion}

\input{main_text/checklist}

% \input{main_text/acknowledgement}

\bibliography{iclr2026_conference}
\bibliographystyle{iclr2026_conference}
\newpage
\appendix
\input{main_text/appendix}

\end{document}

%% file: math_commands.tex
\usepackage{amsmath,amsfonts,bm}

\def\eqref#1{equation~\ref{#1}}
\def\1{\bm{1}}

\DeclareMathAlphabet{\mathsfit}{\encodingdefault}{\sfdefault}{m}{sl}
\SetMathAlphabet{\mathsfit}{bold}{\encodingdefault}{\sfdefault}{bx}{n}

%% file: main_text/abstract.tex
\begin{abstract}

A central paradox in fine-tuning Large Language Models (LLMs) with Reinforcement Learning with Verifiable Reward (RLVR) is the frequent degradation of multi-attempt performance (Pass@k) despite improvements in single-attempt accuracy (Pass@1). This is often accompanied by catastrophic forgetting, where models lose previously acquired skills. Despite numerous proposed methods, the community's focus on the standard reverse-KL divergence has led to a surprising oversight: the potential of alternative f-divergences as a proactive solution has been largely unexamined. We argue that standard RLVR objectives—both those using the mode-seeking reverse-KL divergence and those forgoing a divergence term entirely—lack a crucial mechanism for knowledge retention. The reverse-KL actively accelerates this decay by narrowing the policy, while its absence provides no safeguard against the model drifting from its diverse knowledge base. We propose a fundamental shift in perspective: using the divergence term itself as the solution. Our framework, Diversity-Preserving Hybrid RL (DPH-RL), leverages mass-covering f-divergences (like forward-KL and JS-divergence) to function as a `rehearsal mechanism'. By continuously referencing the initial policy, this approach forces the model to maintain broad solution coverage. Math and SQL generation experiments show that DPH-RL surpasses the GRPO baseline by improving both in-domain Pass@1 and Pass@k scores, and effectively prevents catastrophic forgetting on out-of-domain tasks. Additionally, DPH-RL is more training-efficient because it computes f-divergence using generator functions, requiring only sampling from the initial policy and no online reference model. Our work highlights a crucial, overlooked axis for improving RLVR, demonstrating that the proper selection of a divergence measure is a powerful tool for building more general and diverse reasoning models. Code and data are available at: \url{https://github.com/seamoke/DPH-RL}.

\end{abstract}

%% file: main_text/introduction.tex
\begin{figure}[htbp]
  \centering
  \includegraphics[width=1.0\textwidth]{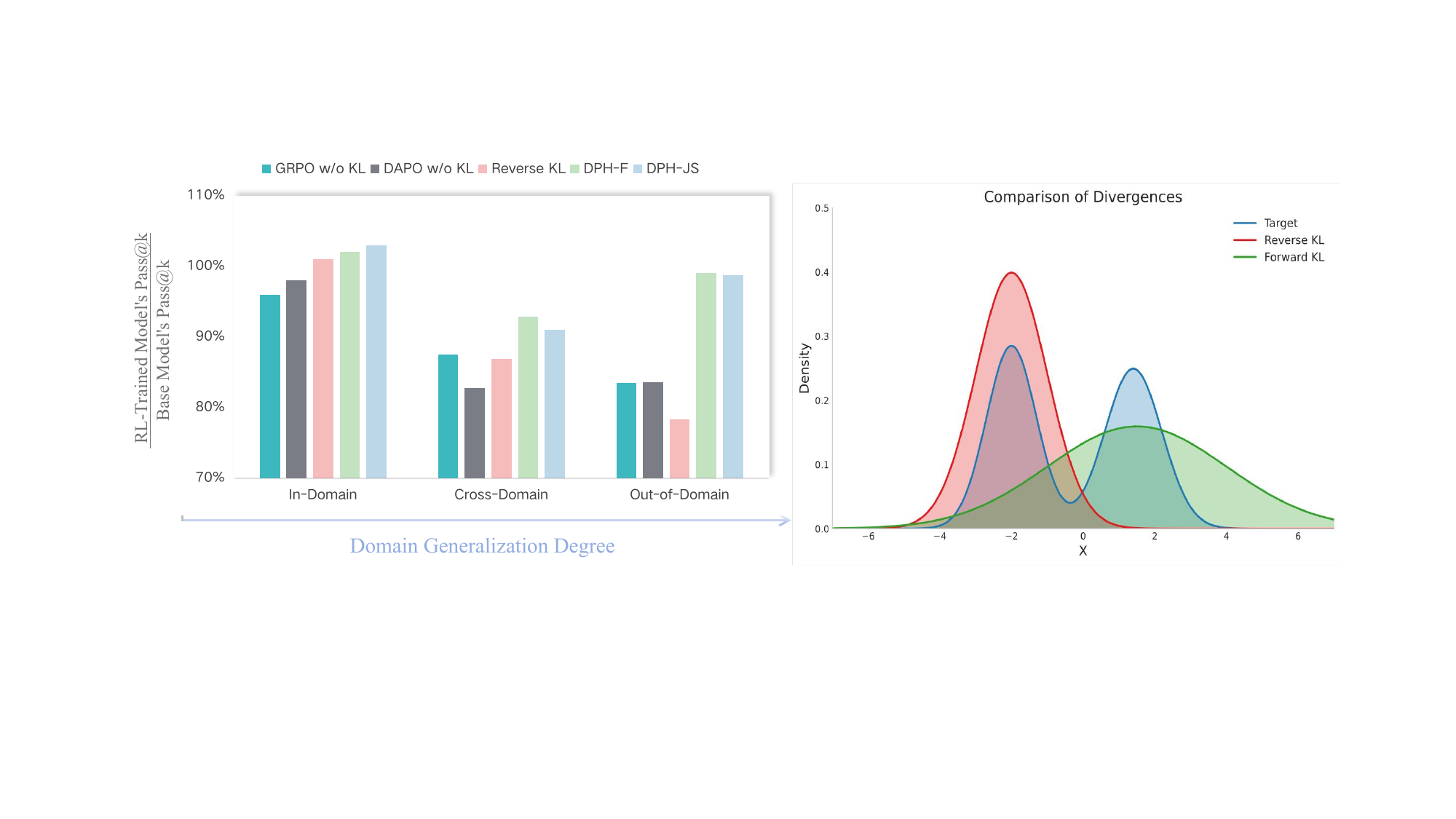}
  \caption{The left panel evaluates the performance gap in Pass@k between the RL-trained model and the Base Model across test sets with varying degrees of divergence from the training data. The right panel visualizes the distributions of reverse-KL and forward-KL.}
  \label{fig:main}
\end{figure}
\vspace{-10pt}

\section{Introduction}
\vspace{-10pt}

Reinforcement Learning with Verifiable Rewards (RLVR) has recently shown significant success in enhancing the mathematical and coding capabilities of Large Language Models (LLMs)~\citep{gpt4,qwen2.5,qwen3,meta2024llama3,guo2025deepseek}. Despite this progress, a critical paradox has emerged: while RLVR-tuned models consistently improve the probability of generating a correct solution in a single attempt (Pass@1), their performance when multiple attempts are permitted (Pass@k) often stagnates or even degrades compared to their base models~\citep{yue2025does}. This discrepancy suggests that RLVR, rather than teaching novel reasoning, may instead overfit the model to known solution paths, thereby sacrificing diversity.

To counteract the narrowing of the model's output distribution~\citep{wang20258020rulehighentropyminority, yue2025does}, which reduces solution diversity, researchers have explored several parallel strategies. One major line of work attributes the phenomenon to entropy collapse and encourages exploration by controlling entropy~\citep{cui2025entropymechanismreinforcementlearning,cheng2025reasoningexplorationentropyperspective,liang2025pass1selfplayvariationalproblem}. Another distinct approach focuses on directly optimizing the Pass@k metric~\citep{mahdavi2025beyond,walder2025pass}, which serves as a more direct proxy for solution diversity than entropy. A third strategy focuses on training setups, such as using extra data~\citep{yan2504learning,dong2025rl}, fine-tuning hyperparameters~\citep{he2025skyworkopenreasoner1,dapo}, or employing a hybrid RL+supervised fine-tuned (SFT) training paradigm~\citep{liu2025acereasonnemotron11advancingmath}. Alongside these methods, a fourth, more fundamental component of the RL objective---the Kullback-Leibler (KL) divergence term used to constrain policy updates---also plays a crucial role. 

However, while the first three avenues have been extensively explored, the influence of the KL divergence term has remained largely under-examined. The community has almost universally adopted the standard \textbf{reverse-KL divergence} ($D_\text{KL}(\pi_{\theta} || \pi_{ref}) = \mathbb{E}_{\pi_{\theta}} \log \frac{\pi_\theta}{\pi_{ref}}$)~\citep{kazemnejad2024vineppo,schulman2017proximal,deepseekmath}, whose well-established \textbf{mode-seeking} nature~\citep{bishop2006pattern} theoretically forces the policy to converge on a single high-probability solution. Our experiments confirm this: a reverse-KL objective yields a single solution style, while a forward-KL objective generates multiple styles (Figure~\ref{fig:demo}). Additionally, we observe that models trained with a reverse-KL term or \emph{no KL} term at all suffer from a decline in Pass@k and catastrophic forgetting. As shown in the left panel of Figure~\ref{fig:main}, both GRPO without KL~\citep{deepseekmath}, DAPO~\citep{dapo}, and Reverse KL show a significant decline in Pass@k performance. In out-of-domain scenarios, the Reverse KL approach performs even worse. These models only correctly answer about 85\% of the problems they were previously able to solve (Figure~\ref{fig:keep_exp}). This phenomenon is a known challenge in sequential learning paradigms with neural networks, not a fundamental flaw of RL, which in principle does not suffer from this in tabular cases \citep{hamadanian2023online}.

Despite these significant issues, prior RLVR methods have almost exclusively relied on the standard \textbf{reverse-KL divergence}, while the exploration of other f-divergences has been largely confined to offline RL or alignment tasks (``f-PO'' and ``f-DPO'')~\citep{wangbeyond,han2025f}. This trend severely underestimates their potential to solve the diversity collapse problem. To bridge this critical gap, we introduce the Diversity-Preserving Framework (DPH-RL), an approach that employs different f-divergences to preserve model diversity within the online RLVR framework. A straightforward example is the \textbf{forward-KL divergence} ($D_\text{KL}(\pi_{ref} || \pi_{\theta})=\mathbb{E}_{\pi_{ref}} \log \frac{\pi_{ref}}{\pi_\theta}$), whose theoretical \textbf{mass-covering} property penalizes the policy for failing to cover all solutions in the reference distribution. From a practitioner's standpoint, the forward-KL objective effectively creates an ``anchor dataset", forcing the model to continuously rehearse its original knowledge base---a mechanism that mirrors human learning and prevents the catastrophic forgetting seen with reverse-KL. We generalize this concept to the broader family of \textbf{f-divergences}, including Jensen-Shannon (JS) divergence and $\alpha$-divergence (see Table~\ref{tab:f_divergences}). Our experiments demonstrate that these mass-covering divergences, specifically forward-KL and JS-divergence, lead to significant improvements in both Pass@1 and Pass@k, maintain strong out-of-domain performance, and achieve these gains without external models. Notably, this approach is \textbf{orthogonal} to existing methods that focus on entropy control or reward shaping.

We make three main contributions:\\
1. Systematic Analysis of Diversity Collapse: Focusing on the KL divergence term, we provide a systematic analysis of the solution diversity collapse in RLVR, identifying the standard reverse-KL divergence as a primary cause. We show that its mode-seeking nature not only suppresses Pass@k performance—often to levels below the base model—but also exacerbates catastrophic forgetting and leads to poor out-of-domain generalization. \\
2. A Novel DPH-RL Framework: We reframe the role of the KL divergence, proposing its use not as a mere policy constraint but as an active diversity-preserving mechanism. Based on this principle, we introduce DPH-RL, a novel framework that employs mass-covering f-divergences (e.g., Forward-KL and JS-divergence) to serve as a rehearsal mechanism. It effectively outperforms GRPO in both Pass@1 and Pass@k, while mitigating the Pass@k degradation relative to the base model.\\
3. Extensive Empirical Validation: Through extensive experiments on a range of models (Llama and Qwen, 7B to 32B) and complex reasoning tasks (mathematics and SQL), we demonstrate the robustness and superiority of DPH-RL. Our method consistently outperforms prior work on both in-domain and out-of-domain benchmarks, successfully mitigating the trade-off between greedy performance and solution diversity.

% Analyze Diversity Collapse: We identify the mode-seeking nature of reverse-KL as a primary cause of diversity collapse, catastrophic forgetting, and poor OOD generalization in RLVR.

% Propose DPH-RL Framework: We introduce DPH-RL, which reframes the divergence term as a rehearsal mechanism, using mass-covering f-divergences to preserve diversity and improve performance.

% Extensive Empirical Validation: We demonstrate DPH-RL's superiority across multiple models and tasks, mitigating the trade-off between greedy performance and solution diversity.

%% file: main_text/related_work.tex
\vspace{-10pt}

\section{Related Work}
\vspace{-10pt}

%Following OpenAI's success~\citep{o1}, models like DeepSeek R1~\citep{guo2025deepseek}, QwQ~\citep{qwq}, Kimi k1.5~\citep{ki1.5}, and Qwen3~\citep{qwen3} have been developed with the goal of matching or exceeding its performance.
A more recent development is RLVR~\citep{yue2025does}, a promising strategy for boosting LLM reasoning, especially in areas such as mathematics, coding, and analysis~\citep{deepseekmath, guo2025deepseek, qwen2.5, lambert2025tulu3pushingfrontiers, htl,shen2025aienhanced,o1,ki1.5,qwen3}. To address diversity collapse in RLVR, one major line of research focuses on either directly controlling model entropy~\citep{cui2025entropymechanismreinforcementlearning,prabhudesai2025maximizingconfidenceimprovesreasoning,agarwal2025unreasonableeffectivenessentropyminimization,li2025curecriticaltokenguidedreconcatenationentropycollapse,cheng2025reasoningexplorationentropyperspective,liang2025pass1selfplayvariationalproblem} or strategically capturing high-entropy tokens~\citep{wang20258020rulehighentropyminority,chen2025seedgrposemanticentropyenhanced,minimax2025minimaxm1scalingtesttimecompute}. However, some research suggests that the entropy of specific tokens is more strongly correlated with model diversity than sentence-level entropy~\citep{wang20258020rulehighentropyminority,cheng2025reasoningexplorationentropyperspective,deng2025decomposingentropyperformanceexchangemissing,chen2025seedgrposemanticentropyenhanced}; consequently, other methods maintain training diversity by controlling the model's Pass@k metric~\citep{rl_plus,zhang2025count}. From a training design perspective, methods such as dynamically adjusting rollout-related hyperparameters~\citep{Polaris2025,he2025skyworkopenreasoner1,dapo} or using a hybrid SFT+RL~\citep{liu2025acereasonnemotron11advancingmath,liu2025uftunifyingsupervisedreinforcement,liu2025ghpoadaptiveguidancestable} training paradigm are also employed to prevent diversity collapse. It is worth noting that theoretically, DPH-F can be viewed as a divergence-based Experience Replay, sharing a similar philosophy with some existing work~\cite{ouyang2022training,zhang2025rlep}, while DPH-JS extends this into a \textit{Regularized Rehearsal} mechanism via dual-sided sampling. Please refer to Appendix~\ref{app:replay_connection} for details.

Our work is most closely related to recent studies that explore f-divergences for policy optimization. Notably, \cite{wang2024beyond, han2025f} \textbf{replace the reverse-KL divergence in the offline DPO objective with a generalized f-divergence, deriving a new family of ``f-PO'' algorithms}. However, our work is distinct: unlike f-PO's focus on offline preference alignment, we utilize online RLVR with verifiable rewards. Our objective is to address Pass@k diversity collapse through mass-covering f-divergences rather than human preference alignment.

%% file: main_text/Preliminaries.tex
\section{Preliminaries}

\subsection{$f$-divergence}
In information theory, an \textit{f-divergence} is a function $D_f(p\|q)$ that measures the difference between two probability distributions $p$ and $q$. Given a convex function $f: \mathbb{R}^+ \to \mathbb{R}$ such that $f(1) = 0$, the f-divergence is defined as: $$
D_f(p\|q) = \int q(x) f\left(\frac{p(x)}{q(x)}\right) dx
$$. 
This general form allows for the unification of many common divergence measures under a single theoretical framework. The condition $f(1)=0$ ensures that $D_f(p\|q)=0$ if and only if $p=q$. Many common divergences are special cases of f-divergences, corresponding to a specific choice for the generator function $f$. For instance, KL divergence and JS divergence are all instances of f-divergences. The generator function for JS divergence is $(u \log u - (u+1)\log\frac{u+1}{2})/2 $. A summary of some prominent f-divergences and their corresponding generator functions is provided in ~\ref{tab:f_divergences}.

\subsection{Markov Decision Process}

We consider a discounted Markov Decision Process (MDP) defined by the tuple $(\mathcal{S}, \mathcal{A}, P, r, \rho_0, \gamma)$, where $\mathcal{S}$ is the state space, $\mathcal{A}$ is the action space, $P(s'|s, a)$ is the transition probability, $r(s)$ is the reward function, $\rho_0$ is the initial state distribution, and $\gamma \in [0, 1)$ is the discount factor. A stochastic policy $\pi(a|s)$ defines a probability distribution over actions for each state. The goal is to maximize the performance objective, i.e., the expected cumulative discounted reward:
\begin{equation}
J(\pi) = \mathbb{E}_{\tau \sim \pi} \left[ \sum_{t=0}^{\infty} \gamma^t r(s_t) \right], \quad \text{where} \quad
s_0 \sim \rho_0(s_0), \ a_t \sim \pi(a_t|s_t), \ s_{t+1} \sim P(s_{t+1}|s_t, a_t).
\end{equation}

A seminal result from \cite{kakade2002approximately} expresses the performance of a policy $\pi$ in terms of an older policy $\pi_{\text{old}}$ using the advantage function $A_{\pi_{\text{old}}}(s, a)$:$
J(\pi) = J(\pi_{\text{old}}) + \mathbb{E}_{\tau \sim \pi} \left[ \sum_{t=0}^{\infty} \gamma^t A_{\pi_{\text{old}}}(s_t, a_t) \right].$
This can be rewritten using the discounted state visitation distribution, $\rho_{\pi}(s) \triangleq \sum_{t=0}^\infty \gamma^t P(s_t=s | \pi)$, which gives the probability distribution over states encountered under policy $\pi$:
\begin{equation}
J(\pi) = J(\pi_{\text{old}}) + \sum_{s \in \mathcal{S}} \rho_{\pi}(s) \sum_{a \in \mathcal{A}} \pi(a|s) A_{\pi_{\text{old}}}(s, a).
\end{equation}
Directly optimizing this expression is difficult because the state distribution $\rho_{\pi}$ depends on the new policy $\pi$. Therefore, algorithms like TRPO and PPO optimize a surrogate objective $L_{\pi_{\text{old}}}(\pi)$ by approximating $\rho_{\pi}$ with the distribution from the old policy, $\rho_{\pi_{\text{old}}}$:
\begin{equation}
L_{\pi_{\text{old}}}(\pi) = J(\pi_{\text{old}}) +  \sum_{s \in \mathcal{S}} \rho_{\pi_{\text{old}}}(s) \sum_{a \in \mathcal{A}} \pi(a|s) A_{\pi_{\text{old}}}(s, a).
\end{equation}
This approximation is reliable when $\pi$ is close to $\pi_{\text{old}}$, forming the basis of modern policy gradient methods.

%% file: main_text/method.tex
\section{Method}

In this section, we detail our implementation of reinforcement learning (RL) with an f-divergence regularizer. Our objective is to find the optimal policy for the following optimization problem:

\begin{equation}\label{equ:general_objective}
\max_{\pi_\theta}\mathbb{E}_{q \sim \mathcal{D}}\left[\mathbb{E}_{a\sim \pi_{\theta}(\cdot|q)}[r(a|q)] - \eta D_{f}(\pi_{\theta}(\cdot|q)||\pi_{ref}(\cdot|q) )\right].
\end{equation}

In this formulation, $\pi_{ref}$ serves as the initial reference policy, such as a base or SFT model. The term $r(a|q)$ represents the reward for taking action $a$ in response to query $q$, and the hyperparameter $\eta$ controls the penalty for deviating from $\pi_{ref}$. While standard algorithms like PPO can optimize the reward term, our discussion focuses on estimating the f-divergence term. A uniform application of this approach across all queries in dataset $\mathcal{D}$ can be suboptimal. For queries where $\pi_{ref}$ performs well, aggressive reward maximization is unnecessary and risks degrading performance. Conversely, for challenging queries where $\pi_{ref}$ struggles, the f-divergence term can overly constrain the policy $\pi_{\theta}$, limiting its ability to explore better solutions. To address this, we propose a targeted strategy that partitions the dataset $\mathcal{D}$ into two subsets. The methodology is divided into two phases: a \textbf{pre-sampling stage} and an \textbf{online training stage}.

\subsection{Pre-sampling Stage}

We consider a specific instance of f-divergence, \textbf{Forward-KL}, defined as:
\begin{equation*}
D_{\text{forward-KL}}( \pi_{\theta} || \pi_{ref}) \triangleq D_{\text{KL}}(\pi_{ref} || \pi_{\theta}) = \mathbb{E}_{a \sim \pi_{ref}} [\log(\pi_{ref}(a|q)) - \log(\pi_\theta(a|q))].
\end{equation*}
To facilitate computing this expectation, we adopt a pre-sampling strategy. Before training $\pi_\theta$, we partition the dataset $\mathcal{D}$. For each query $Q$, we generate and evaluate $k$ independent samples. Based on a correctness threshold, each query is classified as either ``\textbf{near-perfect}'' or ``\textbf{exploration}," splitting $\mathcal{D}$ into $\mathcal{D}_{\text{pef}}$ and $\mathcal{D}_{\text{exp}}$. This partition allows our agent to focus on challenging examples in $\mathcal{D}_{\text{exp}}$ while using a KL divergence constraint to maintain performance on mastered examples in $\mathcal{D}_{\text{pef}}$. To mitigate sampling bias, we further refine $\mathcal{D}_{\text{pef}}$. For each query in this subset, we draw one final sample. We only retain the query in $\mathcal{D}_{\text{pef}}$ if this new sample is correct; otherwise, we discard the query or move it to $\mathcal{D}_{\text{exp}}$.

\subsection{Online Training Stage}

In the online training stage, we simultaneously train the model using two distinct loss functions tailored for the $\mathcal{D}_{\text{exp}}$ and $\mathcal{D}_{\text{pef}}$ datasets. For samples from $\mathcal{D}_{\text{exp}}$, we want the model to have maximum freedom for exploration. Conversely, for samples from $\mathcal{D}_{\text{pef}}$, we want the model to retain its capabilities. To achieve this, we employ two key f-divergences: \textbf{Forward-KL} and \textbf{JS divergence}. Our method is thus divided into two approaches: \textbf{DPH-F} for forward-KL and \textbf{DPH-JS} for JS divergence. Forward-KL divergence, as we define it, penalizes instances where the reference policy $\pi_{ref}$ assigns a high probability to an action, but the new policy $\pi_{\theta}$ assigns a near-zero probability. This property encourages the new policy $\pi_{\theta}$ to maintain coverage of all modes of the reference policy $\pi_{ref}$, thereby preserving its original diversity. JS divergence provides a symmetric and more stable alternative to KL divergence. It encourages the new policy $\pi_{\theta}$ to maintain high similarity with the reference policy $\pi_{ref}$ while achieving high performance, effectively preventing policy collapse.

\subsubsection{Generator-Based Implementation}

Our implementation relies on pre-sampling from the reference policy. This approach allows us to compute the divergence term using a static dataset, eliminating the need to run inference with the reference model during the online training loop.

\paragraph{Loss for $\mathcal{D}_{\text{exp}}$}

For challenging samples in $\mathcal{D}_{\text{exp}}$, we remove the KL divergence penalty from the loss function entirely. This allows the model to perform pure policy optimization based solely on the reward signal, enabling more aggressive exploration. Specifically, the loss function for these samples is the standard PPO-clip objective:

\begin{equation}
\label{eq:dphrl_exp}
\begin{split}
\mathcal{L}_{\text{DPH-exp}}(\theta) ={}& -\mathbb{E}_{\substack{q\sim \mathcal{D}_{\text{exp}} \\ o_i\sim \pi_{\theta_{\text{old}}}(\cdot\mid q)}} \Bigg[ \frac{1}{G}\sum_{i=1}^{G} \frac{1}{|o_i|}\sum_{t=1}^{|o_i|} \min \left( \rho_{i,t}(\theta) \hat{A}_{i,t}, \right. \\
& \qquad \left. \text{clip} \left( \rho_{i,t}(\theta), 1 - \varepsilon, 1 + \varepsilon \right) \hat{A}_{i,t} \right) \Bigg]
\end{split}
\end{equation}
Please refer to Appendix \ref{appendix:rlvr algo} for the definitions of the symbols in the formula.

\paragraph{Loss for $\mathcal{D}_{\text{pef}}$}
For all f-divergences, the general formula for calculating the loss on $\mathcal{D}_{\text{pef}}$ is:
\begin{equation}
\label{eq:fkl_pef_penalty}
\mathcal{L}_{\text{pef}}(\theta) = \mathbb{E}_{q \sim \mathcal{D}_{\text{pef}}} \left[D_f(\pi_{\theta}||\pi_{ref}\right)].
\end{equation}

For DPH-F, the loss term is:
$\mathbb{E}_{q \sim \mathcal{D}_{\text{pef}}} \left[\sum_a \pi_{ref}(a|q)\log\left(\frac{\pi_{ref}(a|q)}{\pi_{\theta}(a|q)}\right) \right]$, where a represents a single response sampled once from the reference policy $\pi_{ref}(.|q)$. For DPH-JS, the loss term is:
$\mathbb{E}_{q \sim \mathcal{D}_{\text{pef}}} \left[ \sum_a \pi_{ref}(a|q) \left( \frac{u \log u}{2} - \frac{u+1}{2} \log\left(\frac{u+1}{2}\right) \right) \right]$, where $u= \pi_{\theta}(a|q)/\pi_{ref}(a|q)$.
\subsubsection{Overall Loss Function}

The total loss for a given batch is a combination of these two objectives. For a mixed batch, we determine whether the data comes from $\mathcal{D}_{\text{pef}}$ or $\mathcal{D}_{\text{exp}}$ and then calculate the corresponding loss:
\begin{equation}
\mathcal{L}_{\text{DPH-RL}}(\theta) = \mathcal{L}_{\text{exp}}(\theta) + \eta \mathcal{L}_{\text{pef}}(\theta),
\end{equation}

\subsection{Enhanced Monotonic Improvement Guarantee}
In this section, we derive an enhanced monotonic improvement guarantee for TRPO-style algorithms that leverage our proposed method. Our framework is built upon two distinct datasets: $\mathcal{D}_{\text{pef}}$, which stores near-perfect reasoning trajectories, and $\mathcal{D}_{\text{exp}}$, an exploration set for tasks the model has not yet mastered. Central to our analysis is a conditional reference policy, $\pi_{\text{pef}}$, whose definition depends on the data source. For states associated with the \textbf{near-perfect dataset ($\mathcal{D}_{\text{pef}}$)}, $\pi_{\text{pef}}$ is the policy induced by the stored correct trajectories. The f-divergence regularization term $D_f(\pi \Vert \pi_{\text{pef}})$ pulls the learned policy $\pi$ towards these expert solutions. For states associated with the \textbf{exploration dataset ($\mathcal{D}_{\text{exp}}$)}, $\pi_{\text{pef}}$ is defined as the current policy $\pi$. This makes the regularization term $D_f(\pi \Vert \pi)$ equal to zero, thereby disabling it and permitting unrestricted exploration. To analyze the effect of this conditional regularization, our derivation hinges on the following mild assumption.

\begin{assumption}
\label{assump:pef_advantage}
For any policy $\pi$ encountered during training, there exists a constant $\delta \geq 0$ such that for any state $s$, the expected advantage of actions from the reference policy $\pi_{\text{pef}}$, evaluated with respect to $\pi$, is lower-bounded by $\delta$: $
\mathbb{E}_{a_{\text{pef}} \sim \pi_{\text{pef}}(\cdot|s)}[A^{\pi}(s, a_{\text{pef}})] \geq \delta.
$
\end{assumption}

This assumption is well-founded. In regions of the state space well-represented by the near-perfect dataset $\mathcal{D}_{\text{pef}}$, the actions from $\pi_{\text{pef}}$ are expected to be superior, yielding a positive advantage ($\delta > 0$). Conversely, for states associated with $\mathcal{D}_{\text{exp}}$, where our method effectively applies no regularization, the bound trivially holds with $\delta = 0$. Leveraging this assumption, we present the following theorem, which establishes a stronger lower bound on policy improvement than the one in the original TRPO analysis \citep{schulman2015trust}.

\begin{theorem}[Enhanced Monotonic Improvement]
\label{theorem:monotonic_improvement}
Let $\alpha_1 = \max_s D_{\text{KL}}(\pi(\cdot|s) \Vert \pi_{\text{old}}(\cdot|s))$ and $\alpha_2 = \max_s D_{f}(\pi(\cdot|s) \Vert \pi_{\text{pef}}(\cdot|s))$, where the divergence $D_f$ can be the forward-KL, $\alpha$-divergence, or Jensen-Shannon divergence. If Assumption~\ref{assump:pef_advantage} holds, then the following bound on policy improvement is guaranteed:
\begin{equation}
\label{eq:monotonic_bound}
J(\pi) - L_{\pi_{\text{old}}}(\pi) \geq 
- \frac{2\gamma\alpha_1 \epsilon_{\pi} }{(1-\gamma)^2} + \epsilon_f 
\end{equation}
where $\epsilon_f = \frac{\delta}{1-\gamma} - \frac{ C_f \gamma\alpha_2  \epsilon_{\text{pef}}}{(1-\gamma)^2}$, $\epsilon_{\pi} = \max_{s,a} |A_{\pi}(s,a)|$ and $\epsilon_{\text{pef}} = \max_{s,a} |A_{\pi_{\text{pef}}}(s,a)|$, $C_f$ is a positive constant depending on the choice of the f-divergence.
\end{theorem}

Comparing with the monotonic improvement theorem in \cite{schulman2015trust}, where the right hand side of the inequality is $-\frac{2\gamma \alpha_1\epsilon_\pi}{(1-\gamma)^2}$, we have an additional term $\epsilon_f$.  On the distribution associated with expert dataset $\mathcal{D}_{\text{pef}}$, where $\delta > 0$ and effective regularization keeps $\alpha_2$ small, $\epsilon_f$ becomes a positive bonus. This results in a strictly better lower bound on policy improvement than in the original TRPO analysis, rigorously justifying that our method accelerates convergence by leveraging known expert behavior. On the exploration set $\mathcal{D}_{\text{exp}}$, the bound gracefully reduces to the standard TRPO guarantee. The proof of Theorem~\ref{theorem:monotonic_improvement} follows a similar line of reasoning to the analyses in \citet{schulman2015trust} and \citet{kang2018policy}. The detailed derivation is deferred to the Appendix \ref{app:proof} for the completeness.

%% file: main_text/experiments.tex
\vspace{-3pt}
\section{Experiments}

\definecolor{c1}{HTML}{DC143C}
\definecolor{mine}{RGB}{205, 232, 248}
\definecolor{c2}{HTML}{32CD32}

In this section, we conduct extensive experiments to demonstrate the effectiveness and generalization of DPH-RL. The experimental setups were applied to two types of tasks: SQL (where the LLM generates SQL code and executes it using a tool) and mathematical reasoning, using both Llama 3.1~\cite{meta2024llama3} and Qwen2.5~\cite{qwen2.5} series models. In our experiments, we partition the dataset into $\mathcal{D}_{\text{exp}}$ and $\mathcal{D}_{\text{pef}}$ based on the capabilities of each model. For the Llama model, we define the $\mathcal{D}_{\text{pef}}$ using a success rate of 6 out of 8 attempts, while for the Qwen model, this threshold is raised to 7 out of 8 attempts. Since our method is orthogonal to existing GRPO variants, we use GRPO, DAPO, and reverse-KL (RKL) as baselines.  Please refer to Appendix~\ref{training_details} for training details.

\subsection{SQL}

\begin{table*}[htpb]
\centering
\caption{Results of \textbf{Llama-3.1-8B-Instruct} on SQL Tasks. The highest scores for RL Models are \colorbox{mine}{highlighted}. Figure~\ref{fig:error_bar} shows the error bars for each method over three training runs.}
\label{TableSQL}
\footnotesize
\begin{adjustbox}{max width=\textwidth}
\begin{tabular}{lcccccc}
\toprule
\multirow{2}{*}{\textbf{Model}} & \multicolumn{3}{c}{\textbf{Bird}} & \multicolumn{3}{c}{\textbf{Spider}} \\
\cmidrule(lr){2-4} \cmidrule(lr){5-7}
& {Greedy} & {Pass@8} & {Pass@16} & {Greedy} & {Pass@8} & {Pass@16} \\
\midrule
Base Model & 42.4 & 68.8 & 75.0 & 71.0 & 90.9 & 93.2 \\
\midrule
\multicolumn{7}{l}{\textit{RL Models}} \\
GRPO & 58.5 & 66.2 & 67.7 & 73.0 & 79.5 & 80.6 \\
DAPO & 60.0 & 67.2 & 69.0 & 71.1 & 75.3 & 76.7 \\
RKL & 60.0 & 69.8 & 71.8 & 71.0 & 79.0 & 80.6 \\
DPH-F & 60.4 & 70.1 & 71.6 & 77.4 & 84.5 & 85.7 \\
\midrule

\multicolumn{7}{l}{{$\alpha$ Divergence}} \\
\hspace{0.5cm}-$\alpha$=0.2 & 60.1 & 67.0 & 69.1 & 75.4 & 78.8 & 80.1 \\
\hspace{0.5cm}-$\alpha$=0.5 & 60.8 & 68.0 & 69.8 & 75.2 & 80.4 & 82.5 \\
\hspace{0.5cm}-$\alpha$=0.8 & 60.5 & 69.2 & 70.4 & 75.8 & 81.7 & 83.6 \\
\midrule
\multicolumn{7}{l}{{DPH-JS (\textit{Generator})}} \\
\hspace{0.5cm}-$\eta$=0.01 & 59.8 & 66.7 & 68.1 & 76.4 & 79.1 & 80.8 \\
\hspace{0.5cm}-$\eta$=0.05 & 61.3 & 69.5 & 70.9 & 75.3 & 82.7 & 83.8 \\
\hspace{0.5cm}-$\eta$=0.2 & \colorbox{mine}{62.8} & \colorbox{mine}{70.5} & \colorbox{mine}{72.4} & 76.0 & 82.7 & 84.1 \\
\midrule
\multicolumn{7}{l}{{DPH-JS (\textit{Divergence Definition})}} \\
\hspace{0.5cm}-$\eta$=0.01 & 59.6 & 66.6 & 68.2 & 76.7 & 80.5 & 81.5 \\
\hspace{0.5cm}-$\eta$=0.05 & 61.4 & 69.6 & 71.3 & 75.9 & 83.5 & 84.7 \\
\hspace{0.5cm}-$\eta$=0.2 & 62.4 & 70.1 & 71.9 & \colorbox{mine}{78.6} & \colorbox{mine}{85.2} & \colorbox{mine}{86.7} \\
\bottomrule
\end{tabular}
\end{adjustbox}
\end{table*}
% \multicolumn{7}{c}{\textbf{OmniSQL-32B}} \\
% \midrule
% Base Model & 65.0 & 77.3 & 79.8 & 85.0 & 92.8 & 93.6 \\

% RL Models \\ % 添加一个类别标题
% GRPO & 69.4 & 76.3 & 78.8 & 84.7 & 91.6 & 92.4 \\
% DAPO & 69.9 & 76.8 & 79.2 & 83.8 & 91.0 & 91.7 \\
% DPH-F & 70.4 & 78.6 & 80.8 & \colorbox{mine}{86.1} & \colorbox{mine}{92.7} & \colorbox{mine}{93.3} \\
% DPH-JS & \colorbox{mine}{70.5} & \colorbox{mine}{79.2} & \colorbox{mine}{81.9} & 84.9 & 91.9 & 93.1 \\
% \midrule
% \multicolumn{7}{c}{\textbf{OmniSQL-32B(Fine-tuned Model)}} \\
% \midrule
% Base Model & 64.9 & 77.5 & 79.86 & .0 & 90.9 & 93.2 \\

% RL Models \\ % 添加一个类别标题
% GRPO & 58.5 & 66.2 & 67.7 & 73.0 & 79.5 & 80.6 \\
% DAPO & 60.0 & 67.2 & 69.0 & 71.1 & 75.3 & 76.7 \\
% DPH-F & 60.4 & 70.1 & 71.6 & \colorbox{mine}{77.4} & \colorbox{mine}{84.5} & \colorbox{mine}{85.7} \\
% DPH-JS & \colorbox{mine}{62.8} & \colorbox{mine}{70.5} & \colorbox{mine}{72.4} & 76.0 & 82.7 & 84.1 \\
% \bottomrule
% \end{tabularx}
% }
% \end{table*}

As shown in Table~\ref{TableSQL}, in the Bird dataset, the Pass@8 scores for both GRPO and DAPO are lower than the base model, while our DPH-F and DPH-JS methods surpass the base model. This indicates that our strategy possesses a more robust ability to maintain model diversity. Specifically, DPH-JS shows Pass@8 scores that are 4.3\% and 3.3\% higher than GRPO and DAPO, respectively. For the Spider dataset, which is a cross-domain SQL task, the Pass@k metrics for all models generally suffer from performance degradation. However, DPH-F and DPH-JS can maintain accuracy levels closer to the base model. While DAPO performs better than GRPO on the Bird dataset, its performance is more unstable on the cross-domain data. Regarding cross-domain performance preservation, DPH-F demonstrates a more powerful capability, with its Pass@16 scores being 9.0\% higher than DAPO. This suggests that for simple tasks that the model can already handle correctly, preservation of its original capabilities is more crucial than pure exploration. We conducted experiments on \textbf{OmniSQL-32B}, with the results presented in Table~\ref{omnisql-32b}. Its Greedy performance on the Bird dataset surpasses all current open-source models.

\vspace{-3pt}
\paragraph{Ablation Study}
% Add these packages to your LaTeX preamble if you haven't already
% \usepackage{booktabs}
% \usepackage{multirow}
% \usepackage{tabularx} % Or tabular*

Our method introduces a plug-and-play $f$-divergence loss, weighted by a hyperparameter $\eta$, which we analyze in an ablation study. As shown in Table~\ref{TableSQL}, when the value of $\eta$ is minimal, its actual setup and performance approximate that of DAPO trained exclusively on $\mathcal{D}_{\text{pef}}$. Increasing the value of $\eta$ leads to steady growth in the Pass@16 score, demonstrating the effectiveness of our method. Additionally, we compared two implementations of $f$-divergence: the ``Generator'' form and the ``Divergence Definition'' form (detailed in Appendix~\ref{definition}). Although both yield similar performance, the ``Divergence Definition'' form is computationally expensive as it requires resampling from $\pi_\theta$ and an additional reference model to calculate $\pi_{ref}(a|q)$. Consequently, the ``Generator'' form is significantly more time-efficient.

\begin{table}
 \centering 
 \caption{Out-of-Domain Performance. Evaluating SQL-trained models on mathematical tasks. $\alpha=0.2$ indicates the use of $\alpha$-divergence with $\alpha=0.2$.} 
 \label{SQL_OOD} 
 \footnotesize 
 \adjustbox{max width=\linewidth}{ 
 \renewcommand{\arraystretch}{1.1} % Increase the row height 
 \begin{tabular}{lccccccc} 
 \toprule 
 \multirow{2}{*}{\textbf{Model}} & \multicolumn{2}{c}{\textbf{Pass@64}} & \multicolumn{3}{c}{\textbf{Pass@16}} & \textbf{Pass@8} & \multirow{2}{*}{\textbf{Avg}} \\ 
 \cmidrule(lr){2-3} \cmidrule(lr){4-6} \cmidrule(lr){7-7} 
 & AIME24  & AMC23 & Math500 & Olympiad & Minerva & College Math \\ 
 \midrule 
 \multicolumn{8}{c}{\textbf{Llama-3.1-8B-Instruct (SQL-trained)}} \\ 
 \midrule 
 Base Model & 40.0  & 95.0 & 81.2 & 46.4 & 54.0 & 45.5 & 60.35 \\ 
\multicolumn{7}{l}{\textit{RL Models}} \\
 \quad-GRPO & 33.3  & 77.5 & 72.0 & 41.8 & 51.2 & 38.4 & 52.37 \\ 
 \quad-DAPO & 30.0  & 77.5 & 72.8 & 44.4 & 52.3 & 38.8 & 52.63 \\ 
 \quad-RKL & 23.3  & 72.5 & 70.8 & 33.8 & 49.2 & 41.1 & 48.45 \\ 
 \quad-$\alpha$=0.2 & 40.0 & 90.0 & 79.8 & 44.0 & 51.8 & 41.8 & 57.90 \\ 
 \quad-$\alpha$=0.5 & 36.7  & 85.0 & 80.2 & 43.7 & 50.7 & 43.4 & 56.62 \\ 
 \quad-$\alpha$=0.8 & 43.3  & 87.5 & 80.8 & 43.1 & 54.0 & 44.2 & 58.82 \\ 
 \quad-DPH-F & \colorbox{mine}{46.7} & \colorbox{mine}{95.0} & 80.8 & 46.1 & \colorbox{mine}{54.3} & 43.0 & \colorbox{mine}{60.98} \\ 
 \quad-DPH-JS & 40.0  & 92.5 & \colorbox{mine}{81.2} & \colorbox{mine}{48.2} & 53.8   & \colorbox{mine}{45.7} & 60.23 \\ 
 \midrule 
 \end{tabular} 
 } 
 \end{table}

\vspace{-3pt}
\paragraph{Performance on OOD Tasks} As shown in Table~\ref{SQL_OOD}, we evaluated the diversity of our RL-trained models on five out-of-domain (OOD) mathematical datasets using a Llama model. (We excluded the AIME25 dataset due to its high fluctuation, ensuring a fair comparison.) Without incorporating additional general-purpose data, methods like GRPO and DAPO cause models to overfit to the SQL training domain, leading to a significant drop in performance on other OOD tasks. In contrast, our DPH-F and DPH-JS methods maintain a higher average performance, surpassing DAPO by 8.35\% and 7.6\%, respectively.

% \begin{wrapfigure}{r}{0.4\textwidth} % Adjusted width for better proportion
%     \centering
%     % Ensure image width matches or is slightly less than wrapfigure width
%     \includegraphics[width=0.98\linewidth]{figures/eta_vs_greedy_bird_log_x_with_baseline.pdf}
%     % You can add \vspace{-0.5em} here if the caption feels too far from the image
%     \caption{$\eta$ vs. Greedy\label{fig:rkl_set}} % Label inside caption
%     \vspace{-0.5em}
%     %You might add \vspace{-1em} here if the figure pushes text too far down
% \end{wrapfigure}
\begin{figure}[htbp]
  \centering
  % --- 第一张子图 ---
  \begin{subfigure}[b]{0.61\textwidth}
    \centering
    % 宽度设为 \textwidth 指的是子图宽度的100%
    \includegraphics[width=\textwidth]{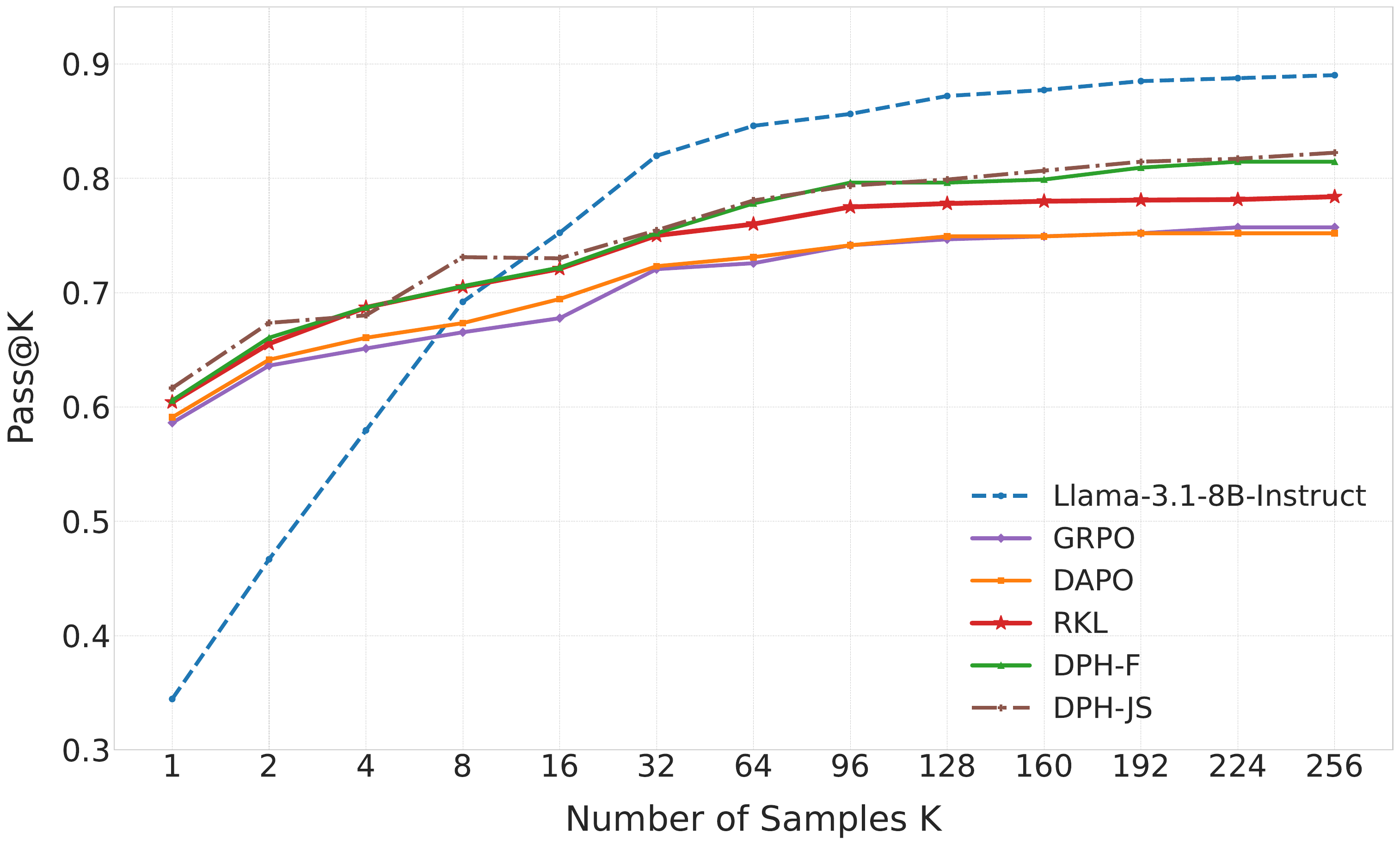}
    \caption{}
    \label{fig:pass256}
  \end{subfigure}
  \hfill % 在两张图之间加入弹性空白，让它们分居两侧
  % --- 第二张子图 ---
  \begin{subfigure}[b]{0.36\textwidth}
    \centering
    \includegraphics[width=\textwidth]{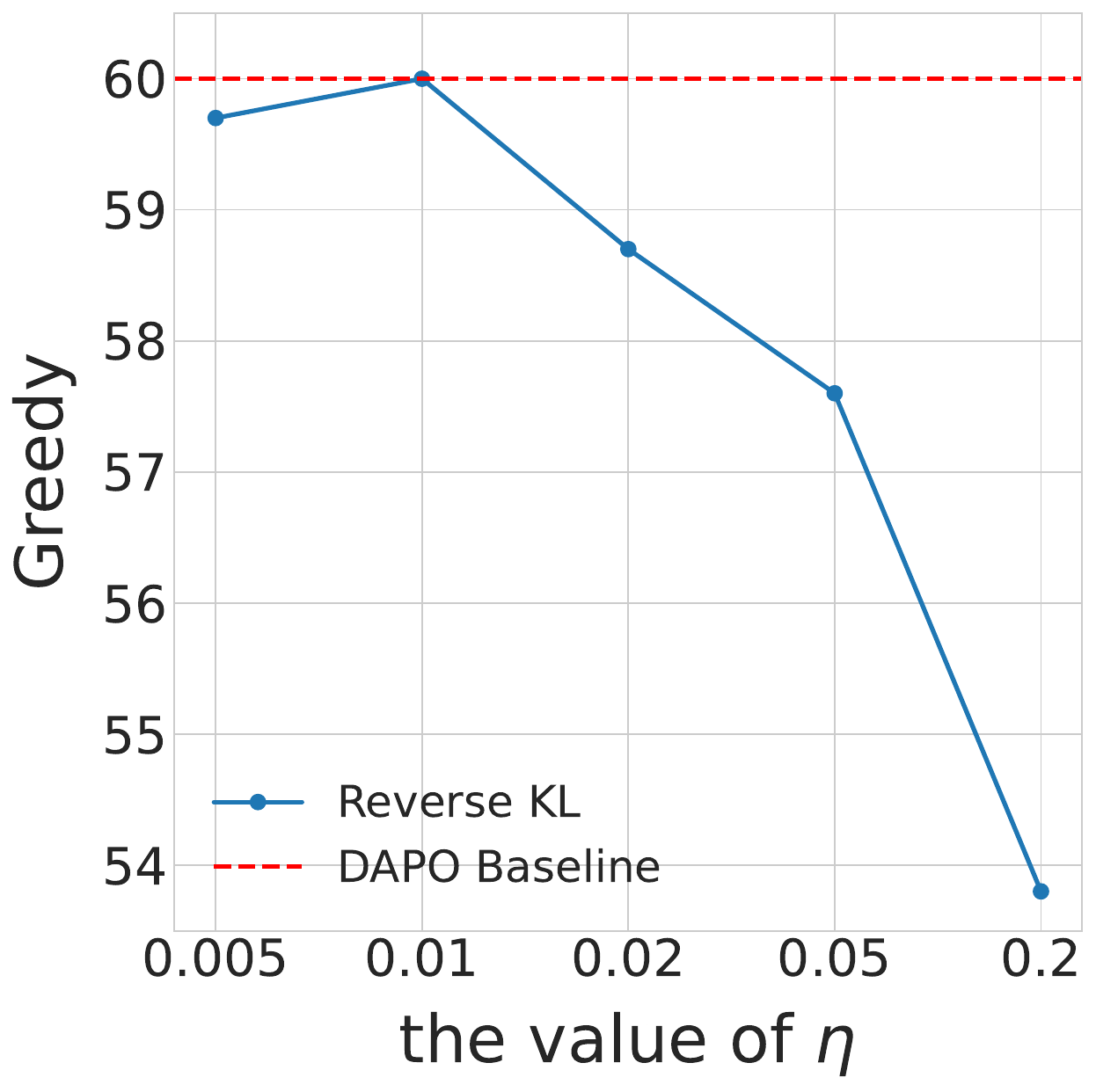}
    \caption{}
    \label{fig:rkl_set}
  \end{subfigure}
  
  % --- 总标题 ---
  \caption{Results of Llama models on the BIRD dataset. (a) Performance of different methods at different pass@k levels. (b) The impact of different KL term parameters on RKL.}
  \label{fig:combined_results}
\end{figure}
\vspace{-6pt}
\paragraph{Performance with Massive Sampling}
We evaluated the performance of various methods as the number of samples increases. As illustrated in Figure~\ref{fig:pass256}, the Pass@$k$ scores for all methods reach convergence by $k=256$. Notably, standard GRPO and DAPO saturate early, plateauing at approximately 75\%. Although DAPO outperforms GRPO at small $k$ values, its mechanism fails to scale effectively to larger $k$. The DPH family of methods significantly outperforms the GRPO baseline. Regarding Reverse KL, while it yields higher performance than GRPO, it remains significantly inferior to the DPH series. Overall, DPH-JS emerges as the superior method across both Pass@1 and Pass@$k$ metrics.
\vspace{-3pt}
\paragraph{Comparisons Under Different $f$-divergences}
In Tables~\ref{TableSQL} and~\ref{SQL_OOD}, we analyzed the impact of different $f$-divergences on the Llama model. First, DAPO and GRPO, which lack any form of $f$-divergence constraint, experienced severe Pass@k collapse on both in-domain and out-of-domain tasks. Conversely, while RKL could maintain a high Pass@k on the training task's test set, its performance collapse on tasks with different distributions was even more severe than methods without any KL penalty. This highlights the limitations of reverse-KL: it causes the model to over-focus on the training data's distribution, thereby completely sacrificing generalization. To validate the reliability of our experimental baseline, we show the performance of RKL with different $\eta$ values in Figure~\ref{fig:rkl_set}. When $\eta > 0.02$, the model's learning performance fails to surpass that of DAPO. This indicates that the chosen $\eta$ value is too large, and therefore our selection of $\eta=0.01$ is well-justified. The $\alpha$ divergence sits between the forward-KL and reverse-KL, effectively preventing the model from veering toward extremes. Across all datasets, it shows a clear trend: as the $\alpha$ value increases, it theoretically approaches the capabilities of forward-KL, while experimentally maintaining a higher Pass@k. In contrast, DPH-F and DPH-JS methods demonstrated strong generalization across different tasks. The DPH-JS method, in particular, not only maintained a higher greedy performance on training tasks but also preserved a high Pass@k value on out-of-domain tasks, which fully demonstrates its significant advantage in preserving a multi-peak distribution.

\subsection{Mathematical Reasoning}

\begin{table*}[ht]
\centering
\caption{The Pass@k metric for models trained on math datasets. In the Llama experiments, the DAPO settings ($\epsilon_\text{high}$=0.28) caused training to crash. To conduct effective experiments, we aligned the settings for all Llama experiments on the Math dataset with GRPO ($\epsilon_\text{high}$=0.2).}
\label{Table3}
\small
\adjustbox{max width=\linewidth}{
\renewcommand{\arraystretch}{1.1} % Increase the row height
\begin{tabular}{lcccccccc}
\toprule
\multirow{2}{*}{\textbf{Model}} & \multicolumn{3}{c}{\textbf{Pass@64}} & \multicolumn{3}{c}{\textbf{Pass@16}} & \textbf{Pass@8} & \multirow{2}{*}{\textbf{Avg}} \\
\cmidrule(lr){2-4} \cmidrule(lr){5-7} \cmidrule(lr){8-8}
& AIME24 & AIME25 & AMC23 & Math500 & Olympiad & Minerva & College Math \\
\midrule
\multicolumn{9}{c}{\textbf{Llama-3.1-8B-Instruct}} \\
\midrule
Base Model & 40.0 & 23.3 & 95.0 & 81.2 & 46.4 & 54.0 & 45.5 & 55.06 \\
RL Models \\
\quad-GRPO & 33.3 & 26.7 & 80.0 & 79.6 & 43.3 & 56.6 & 43.1 & 51.80 \\
\quad-RKL & 36.7 & 16.7 & 75.0 & 80.0 & 39.1 & 56.6 & 43.4 & 49.64 \\
\quad-DPH-F & 36.7 & 26.7 & \colorbox{mine}{90.0} & 80.6 & 44.3 & 57.3 & \colorbox{mine}{45.5} & \colorbox{mine}{54.44} \\
\quad-DPH-JS & \colorbox{mine}{40.0} & 26.7 & 82.5 & \colorbox{mine}{81.4} & \colorbox{mine}{45.8} & \colorbox{mine}{58.1} & 45.1 & 54.23 \\
\midrule
\multicolumn{9}{c}{\textbf{Qwen2.5-Math-7B}} \\
\midrule
Base Model & 63.3 & 56.7 & 87.5 & 88.8 & 61.9 & 56.6 & 42.9 & 65.39 \\
RL Models \\
\quad-GRPO & 56.6 & 50.0 & 97.5 & \colorbox{mine}{93.0} & 62.8 & 64.0 & 50.1 & 67.71 \\
\quad-DAPO & 63.3 & 46.7 & 97.5 & 92.2 & 63.1 & 64.3 & 48.7 & 67.97 \\
\quad-RKL & 66.7 & 40.0 & 97.5 & 92.0 & 64.6 & 64.0 & 51.3 & 68.01 \\
\quad-DPH-F & \colorbox{mine}{73.33} & 50.0 & 97.5 & 92.4 & 63.8 & 64.8 & 50.9 & 70.39 \\
\quad-DPH-JS & 66.7 & \colorbox{mine}{53.3} & \colorbox{mine}{100.0} & 92.8 & \colorbox{mine}{65.2} & \colorbox{mine}{66.2} & \colorbox{mine}{51.0} & \colorbox{mine}{70.74} \\
\bottomrule
\end{tabular}
}
\end{table*}

For the mathematical tasks, we present two types of results, shown in Table~\ref{Table3} and Table~\ref{Table4}, respectively. Table~\ref{Table3} shows the Pass@k scores on various test sets, while Table~\ref{Table4} shows the Mean@k scores, which is the average accuracy calculated by sampling the entire dataset k times. Regarding mathematical ability, different model families have completely different capabilities. For the Llama model, the improvement through RL is very limited; the average mean@k value when using GRPO only increased by 0.93, while Pass@k decreased by 3.26. By contrast, the Qwen model is a more capable model, which allows it to achieve a significant boost in both Mean@k (approximately 20\%) and Pass@k. We selected these two models to explore the performance of the DPH method under different model strengths. For Llama, DPH-JS maintained the original Pass@k value and significantly improved mean@k on AIME, which demonstrates that DPH-JS is a powerful and highly versatile method with strong capabilities in both exploration and preservation. Furthermore, DPH-JS provided a more balanced improvement than GRPO, increasing both Pass@k and mean@k averages, showing it does not sacrifice one metric for the other. For the Qwen model, the overall performance showed a dual trend: a slight decrease in Pass@k on the difficult AIME test set, but an increase on other, simpler datasets. In both scenarios, the DPH-JS method achieved the best performance. On AIME, it maintained a diversity closer to that of the base model compared to both GRPO and DAPO. On the other datasets, it achieved a higher Pass@k. This trend highlights a key challenge for large models in RL: the need to both solidify existing knowledge while also exploring new boundaries. The DPH-JS method successfully balances these two conflicting goals.

%% file: main_text/analysis.tex
\vspace{-6pt}
\section{Analysis}
\vspace{-10pt}
\subsection{Analysis of Keep and Exploration}
\begin{figure}[htbp]
  \centering
  \includegraphics[width=1.0\textwidth]{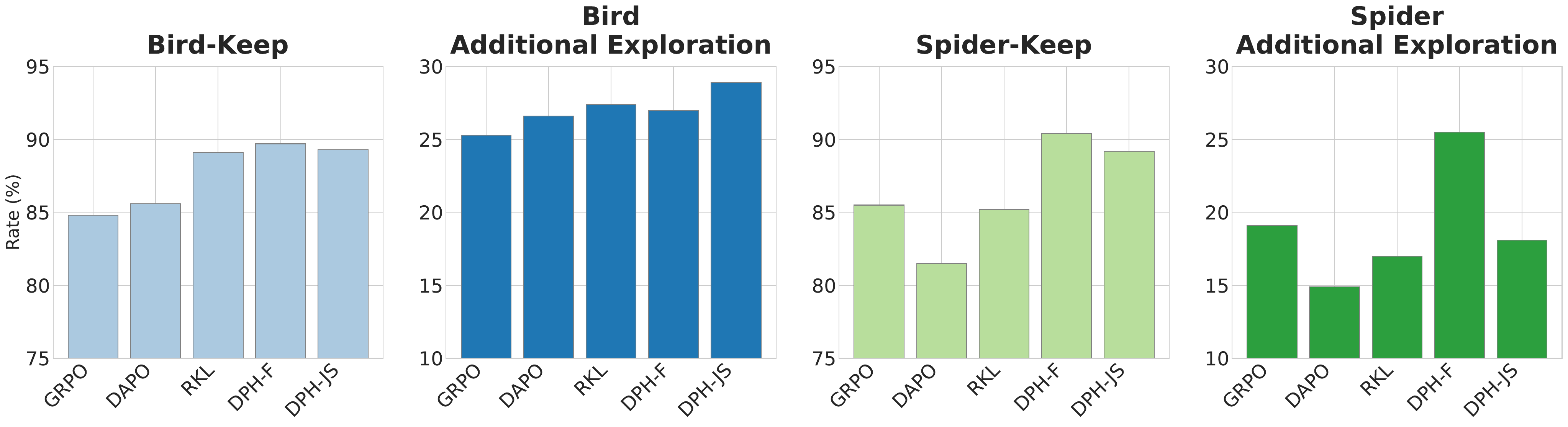}
  \caption{Evaluating RL-Tuned Llama models vs. base model: knowledge retention (keep) vs. new solution discovery (additional exploration).
  }
  \label{fig:keep_exp}
\end{figure}

To gain a more granular understanding of why our method achieves higher Pass@k scores, we extracted correct and incorrect samples from the base model's Pass@8 results. These were designated as data subsets $D_{sub}^{correct}$ and $D_{sub}^{wrong}$, respectively. We then compared them with the model after reinforcement learning. The proportion of samples it got correct on $D_{sub}^{correct}$ is its keep rate, while the proportion it got correct on the incorrect samples from $D_{sub}^{wrong}$ is its additional exploration rate. As shown in Figure~\ref{fig:keep_exp}, the keep rate of GRPO and DAPO both decreased on the two datasets, while the KL method mainly maintained a higher keep rate to maintain a higher Pass@k value. DAPO had higher exploration within the domain than GRPO, but the keep rate outside the domain dropped significantly. RKL could only have a high keep rate on the same Bird test set as the training set, while it dropped to a level similar to GRPO on Spider. DPH-JS had higher exploration on the Bird test set, while DPH-F was higher on the Spider dataset. Furthermore, Appendix~\ref{app:training_progress} shows that DPH is the most stable and effective method during training, while other mainstream approaches either have poor generalization (e.g., DAPO) or diversity collapse late in training (e.g., GRPO and RKL).

\subsection{Impact of Splitting $\mathcal{D}_{\text{exp}}$ and $\mathcal{D}_{\text{pef}}$}
\begin{table*}[htpb]
\centering
\caption{`8 of 8' indicates a threshold where samples correctly solved in 8 out of 8 attempts are assigned to $\mathcal{D}_{\text{pef}}$ in DPH-JS. `Full Data' implies using the entire dataset for both $\mathcal{D}_{\text{pef}}$ and $\mathcal{D}_{\text{exp}}$.}
\label{split}
\footnotesize
\begin{adjustbox}{max width=\textwidth}
% 修改点1: 增加了两列 (lcc cccccc)
\begin{tabular}{lcc cccccc}
\toprule
% 修改点2: 表头增加两列 D_exp 和 D_pef
\multirow{2}{*}{\textbf{Model}} & \multirow{2}{*}{Size of $\mathcal{D}_{\text{exp}}$} & \multirow{2}{*}{Size of $\mathcal{D}_{\text{pef}}$} & \multicolumn{3}{c}{\textbf{Bird}} & \multicolumn{3}{c}{\textbf{Spider}} \\
% 修改点3: cmidrule 范围顺延 (+2)
\cmidrule(lr){4-6} \cmidrule(lr){7-9}
& & & {Greedy} & {Pass@8} & {Pass@16} & {Greedy} & {Pass@8} & {Pass@16} \\
\midrule
% 请在下方 & 之间填入具体数值
GRPO & 8958 & 0& 58.5 & 66.2 & 67.7 & 73.0 & 79.5 & 80.6 \\
DAPO & 8958 & 0& 60.0 & 67.2 & 69.0 & 71.1 & 75.3 & 76.7 \\
8 of 8  & 8266 & 692 & 61.4 & 71.6 & 73.5 & 75.4 & 86.3 & 87.2 \\
7 of 8  & 7429 & 1529 & 61.3 & 71.2 & 73.6 & 75.1 & 85.8 & 87.2 \\
6 of 8   & 6710 & 2248 & \colorbox{mine}{62.8} & 70.5 & 72.4 & \colorbox{mine}{76.0} & 82.7 & 84.1 \\
5 of 8  &5783  &3175 & 60.2 & 70.6 & 74.0 & 74.3 & 86.2 & 88.9 \\
Full Data & 8958 & 8958 & 57.4 & \colorbox{mine}{73.1} & \colorbox{mine}{76.6} & 71.8 & \colorbox{mine}{88.4} & \colorbox{mine}{90.3} \\
\bottomrule
\end{tabular}
\end{adjustbox}
\end{table*}
In the main experiments, we established an accuracy threshold to partition the data into $\mathcal{D}_{\text{exp}}$ and $\mathcal{D}_{\text{pef}}$. This section provides a detailed analysis of the impact of this threshold. As illustrated in Table~\ref{split}, the performance difference between the `7 of 8' and `8 of 8' settings is negligible. At `6 of 8', the model sacrifices some diversity to achieve higher greedy performance; by removing simpler samples from $\mathcal{D}_{\text{exp}}$, the model focuses exploration on more challenging problems, while the KL term ensures continued learning on well-mastered samples. Conversely, at `5 of 8', the significant reduction in $\mathcal{D}_{\text{exp}}$ data leads to a decline in greedy performance, despite yielding improved Pass@$k$ scores. In the `Full Data' setting, greedy performance is poor—even inferior to the GRPO baseline—yet it achieves the highest Pass@$k$. Although the large volume of $\mathcal{D}_{\text{pef}}$ preserves sufficient diversity, `Full Data' suffers from two critical issues: (1) Our goal is to preserve the diversity of correct solutions, not the noise of incorrect ones. Forcing the selection of a single correct response from samples solved correctly only once out of eight attempts introduces statistical bias. (2) Regarding the training objective, the model can trivially mimic the sole correct sample for each query to maximize reward and minimize the KL term, thereby optimizing the loss without genuine learning. In contrast, `8 of 8' represents an unbiased statistical approach requiring no artificial sampling. To compensate for the reduced data scale in this setting, we retain all 8 sampled trajectories for each instance during training. In our experiments, although `8 of 8' allocates only 8\% of the training data to $\mathcal{D}_{\text{pef}}$, it comprehensively outperforms DAPO and GRPO. This demonstrates the robustness of DPH-RL even under extreme conditions, suggesting that directly adopting the `8 of 8' threshold is the most straightforward strategy for its application.

%% file: main_text/conclusion.tex
\vspace{-10pt}
\section{Conclusion}
\vspace{-10pt}

We propose DPH-RL to overcome the limitations of reverse KL divergence in reinforcement learning for LLMs. By leveraging $f$-divergence, our method mitigates both diversity collapse and catastrophic forgetting. We demonstrate that two implementations---a ``Generator'' form that requires no reference model during training and a ``Divergence Definition'' form---yield similar performance. This makes our method highly adaptable, providing enhanced capabilities with computational costs comparable to GRPO. We explicitly recommend the DPH-JS (Generator) variant with the `8 of 8' threshold as the optimal configuration for future researchers. However, a key limitation, shared with GRPO, is the dependence on the base model's capabilities, which become scarce on more challenging datasets. Future work will therefore focus on resolving this trade-off, aiming to develop methods that can both learn effectively from imperfect data and mitigate the resulting estimation bias.

%% file: main_text/checklist.tex
\section*{Reproducibility}

Our datasets are all based on open-source datasets. The experimental methodology, data proportions, and hyperparameter settings are detailed in Appendix~\ref{training_details}. 
%https://anonymous.4open.science/r/DPH-RL-98A0

%% file: main_text/appendix.tex
\section{The Use of Large Language Models}
In our work, we exclusively use LLMs for writing refinement, which means we first write a piece of text ourselves, and then use the LLM to correct grammar, formatting, and other issues. For our experiments, we also use LLMs to help us fix code bugs and generate Python code for plotting.
\section{Multiple Style Capability Experiment}
\begin{figure}[]
  \centering
  \includegraphics[width=1.0\textwidth]{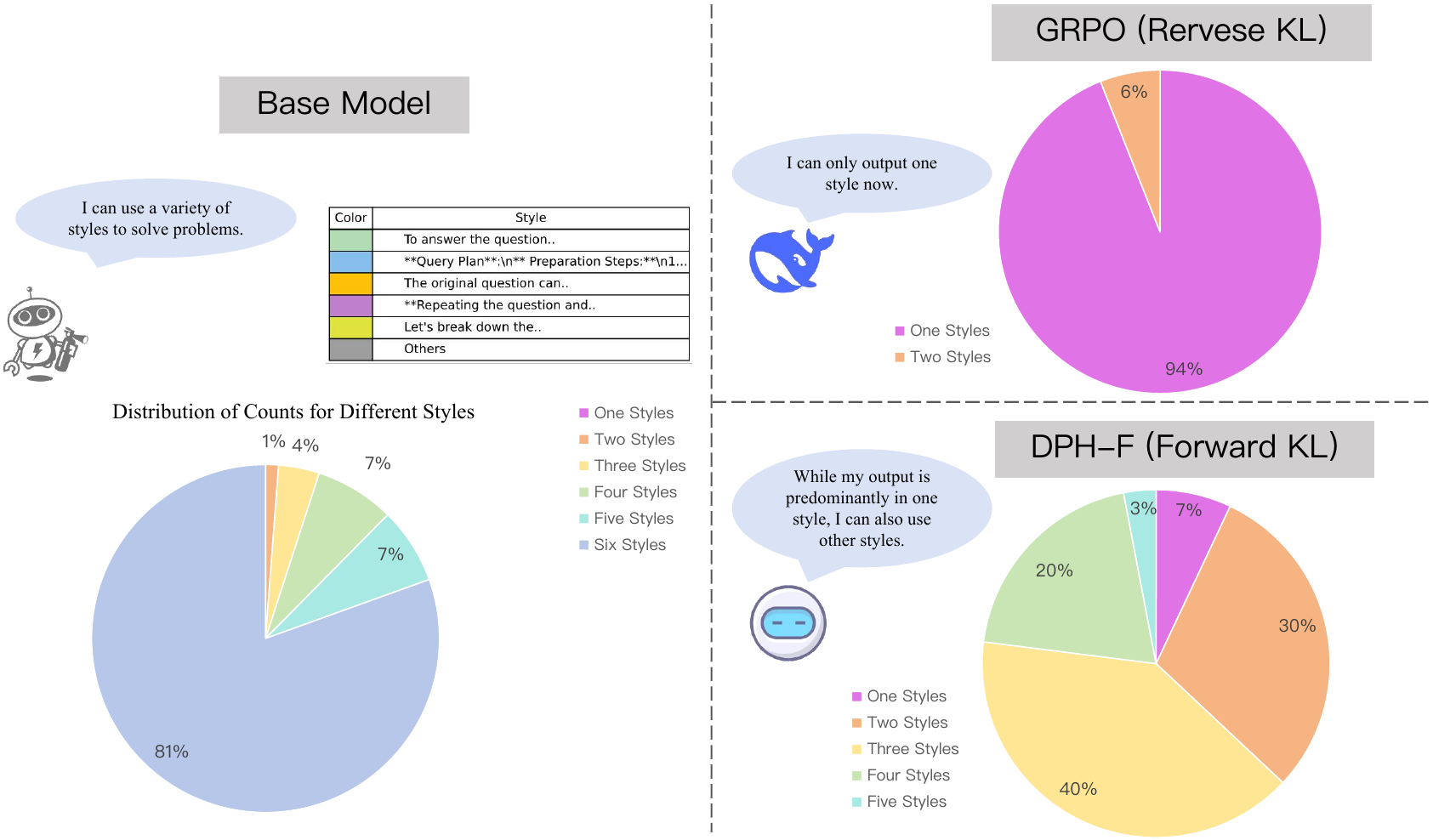}
  \caption{On the left, we construct a base model that outputs multiple solution styles for SQL problems. This model is then used for reinforcement learning training. We calculated the probability of the number of times the model outputted different styles across 32 samples.
  }
  \label{fig:demo}
\end{figure}

We construct a base model that outputs five different solution styles for SQL problems. This model is then used for reinforcement learning training. Next, we sample each question 32 times. Based on the output prefix, we count the number of times each distinct style appears within these 32 samples. For prefixes that are not among our original five styles (e.g., the model combines several different prefixes), we categorize them as ``Others" and treat them as an additional style.

The result shows in Figure~\ref{fig:demo}. On the base model, we find that it can fully output all styles. However, after training with reverse KL-constrained GRPO, most outputs degenerate into a single style. In contrast, forward KL significantly mitigates the degradation of the model's ability to produce different styles.

\section{Preliminaries on $f$-Divergence}
In Table~\ref{tab:f_divergences}, we summarize several typical $f$-divergences and their corresponding generator functions $f(u)$, with the definitions of forward and reverse KL divergence adapted according to the latest research. All listed divergences adhere to the standard properties of the generator function: strict convexity, lower semi-continuity, and the normalization condition $f(1)=0$. Notably, this table clarifies that both forward and reverse KL divergence can be unified as special cases of the $\alpha$-divergence under specific parameters. This provides a unified mathematical framework for measuring and optimizing distribution shifts using $\pi_{\text{ref}}$ in subsequent model training.
\begin{table}[h]
\centering
\caption{Summary of some typical f-divergences $D_f(p\|q)$ together with their generator functions $f(u)$. In this paper, we let $\pi_{ref}$ be $q$, and the formula for f-divergence references an improved version~\citep{wang2024beyond}, the definitions for forward-KL and reverse-KL are swapped compared to Wikipedia. For all divergences, the generator function $f: \mathbb{R}^+ \to \mathbb{R} \cup \{+\infty\}$ is strictly convex, lower-semicontinuous, and satisfies $f(1)=0$. Forward-KL and reverse-KL are the special case of alpha divergence. }
\label{tab:f_divergences}
\footnotesize
\begin{tabularx}{\textwidth}{@{}lXl@{}}
\toprule
\textbf{Name} & \textbf{Divergence Definition $D_f(p\|q)$} & \textbf{Generator $f(u)$} \\
\midrule
Reverse-KL($\alpha=0$) & $\int p(x) \log \frac{p(x)}{q(x)} dx$ & $u \log u$ \\
\addlinespace
Forward-KL($\alpha=1$) & $\int q(x) \log \frac{q(x)}{p(x)} dx$ & $-\log u$ \\
\addlinespace
$\alpha$-divergence ($\alpha \notin \{0,1\}$) & 
\begin{tabular}[t]{@{}l@{}}
$\frac{1}{\alpha(\alpha-1)} \int q(x) \Biggl[ \left( \frac{p(x)}{q(x)} \right)^\alpha$ \\
$\qquad - \alpha \left( \frac{p(x)}{q(x)} \right) - 1 \Biggl] dx$
\end{tabular}
& $(u^{1-\alpha}-(1-\alpha)u-\alpha)/(\alpha(\alpha-1))$ \\
\addlinespace
Jensen-Shannon &
\begin{tabular}[t]{@{}l@{}}
$\frac{1}{2} \int \Biggl( p(x) \log \frac{2p(x)}{p(x)+q(x)}$ \\
$\qquad + q(x) \log \frac{2q(x)}{p(x)+q(x)} \Biggl) dx$
\end{tabular}
& $\frac{u}{2} \log u - \frac{(u+1)}{2}\log\frac{u+1}{2} $ \\
% \addlinespace
% Squared Hellinger & $\int \left(\sqrt{p(x)} - \sqrt{q(x)}\right)^2 dx$ & $(\sqrt{u}-1)^2$ \\
\bottomrule
\end{tabularx}
\end{table}

\section{Theoretical Connection to Experience Replay and Regularized Rehearsal}
\label{app:replay_connection}

In this section, we elaborate on the theoretical interpretation of DPH-RL as a hybrid method rooted in replay mechanisms. We demonstrate that while our method incorporates a replay buffer component, it significantly extends standard Experience Replay (ER) through the introduction of symmetric divergences.

\subsection{DPH-F as Divergence-Based Experience Replay}

Minimizing the Forward-KL divergence, defined as $\mathbb{E}_{x \sim \pi_{ref}} [-\log \pi_{\theta}(x)]$, is mathematically equivalent to maximizing the log-likelihood of trajectories sampled from the reference policy $\pi_{ref}$. In our framework, $\mathcal{D}_{pef}$ serves as a curated buffer containing ``perfect'' trajectories. Consequently, the DPH-F objective acts effectively as an Experience Replay mechanism. By enforcing probability mass coverage on these high-quality solutions, it prevents the model from catastrophically forgetting previously mastered skills while optimizing rewards on the exploration set $\mathcal{D}_{exp}$.

\subsection{DPH-JS as Regularized Rehearsal via Dual-Sided Sampling}

While DPH-F acts as direct rehearsal, DPH-JS (Jensen-Shannon) introduces a more nuanced mechanism. Since the JS divergence is symmetric and defined based on a mixture distribution $M = (\pi_{ref} + \pi_{\theta})/2$, its optimization necessitates \textbf{dual-sided sampling}.

Unlike standard replay which only samples from the historical buffer, minimizing DPH-JS involves comparing samples from two sources:
\begin{itemize}
    \item \textbf{The Reference Buffer ($\pi_{ref}$):} Anchoring the model to the valid distribution to prevent collapse.
    \item \textbf{The Current Policy ($\pi_{\theta}$):} Sampling self-generated trajectories to regulate the model's deviation from the mixture mean.
\end{itemize}

This dual dependency makes DPH-JS a form of \textbf{``Regularized Rehearsal.''} It anchors the model to the reference distribution (similar to ER) to ensure diversity, but simultaneously regulates the model's own exploration. This prevents the model from drifting too far into unstable regions without forcing it to strictly ``memorize'' the buffer as rigidly as Forward-KL or SFT might.

\section{RLVR Algorithms}
\label{appendix:rlvr algo}
\subsection{Group Relative Policy Optimization (GRPO)}

GRPO presents an innovative approach to policy learning that distinguishes itself from methods like Proximal Policy Optimization (PPO) by eliminating the need for an explicit value function. Instead, GRPO computes the advantage in a group-relative manner, offering a streamlined yet effective optimization strategy.

For a specific question-answer pair $(q,a)$, GRPO's underlying behavior policy, $\pi_{\theta_\text{old}}$, generates a group of $G$ individual responses, denoted as $\{ o_i\}_{i=1}^G$. The advantage for the $i$-th response within this ensemble is then precisely calculated by normalizing the rewards specific to that group, $\{ R_i \}_{i=1}^G$:
\begin{equation}
\hat{A}_{i,t} = \frac{r_i - \text{mean}(\{R_i\}_{i=1}^G)}{\text{std}(\{R_i\}_{i=1}^G)}.
\end{equation}

Similar to the clipped surrogate objective found in PPO, GRPO also incorporates a clipping mechanism to constrain policy updates. This helps maintain training stability and improve sample efficiency by ensuring that new policies don't deviate too drastically from previous ones. Beyond this, GRPO further enhances regularization by directly adding a Kullback-Leibler (KL) divergence penalty term to its objective function. This penalty helps prevent the policy from drifting too far from a reference policy, promoting stable and controlled learning.

The comprehensive objective function for GRPO is articulated as:
\begin{equation}
\label{eq:grpoloss_compact}
\begin{aligned}
\mathcal{L}\text{GRPO}(\theta) = -\mathbb{E}_{\substack{(q,a)\sim \mathcal{D} \\ {o_i}\sim \pi_{\theta_{\text{old}}}(\cdot\mid q)}}
\Bigg[ & \frac{1}{G}\sum_{i=1}^{G} \frac{1}{|o_i|}\sum_{t=1}^{|o_i|} \Bigg( \\
& \quad \min \left( \rho_{i,t}(\theta) \hat{A}_{i,t},
\ \text{clip} \left( \rho_{i,t}(\theta), 1 - \varepsilon, 1 + \varepsilon \right) \hat{A}_{i,t} \right) \\
& \quad - \eta D_{\text{KL}}(\pi_{\theta} || \pi_{\text{ref}})
\Bigg) \Bigg],
\end{aligned}
\end{equation}
where $\eta$ is a coefficient for the KL penalty, $D_{\text{KL}}(\pi_{\theta} || \pi_{\text{ref}})$ quantifies the KL divergence between the current policy $\pi_\theta$ and a specified reference policy $\pi_\text{ref}$. The term $\rho_{i,t}(\theta)$ represents the importance sampling ratio for the $i$-th response at time $t$, which is defined as:
\begin{equation}
\rho_{i,t}(\theta)=\frac{\pi_{\theta}(o_{i,t} \mid q, o_{i,<t})}{\pi_{\theta_{\text{old}}}(o_{i,t} \mid q,o_{i,<t})}.
\end{equation}

A critical design choice in GRPO is its sample-level objective computation. Specifically, the loss is first averaged within each generated sequence, and subsequently, these sequence-level losses are averaged across various samples. This distinct computational approach, particularly when compared to token-level optimizations, can significantly influence the algorithm's empirical performance.

Traditional GRPO applies a single, unified loss function to all training samples. This loss includes a KL divergence penalty, typically in the form of reverse KL divergence, which aims to keep the fine-tuned policy close to the base policy. This penalty is added directly to the policy gradient loss. When applied to our experimental setup, the KL penalty for samples from $\mathcal{D}_{\text{exp}}$ would be:

\begin{equation}
\label{eq:grpo_kl_penalty}
\mathcal{L}_{\text{KL-GRPO}} = \mathbb{E}_{q \sim \mathcal{D}_{\text{exp}}} \left[ -\beta D_{\text{KL}}(\pi_{\theta}|| \pi_{\text{base}}) \right],
\end{equation}

where the reverse KL divergence is defined as:
\begin{equation}
\label{eq:reverse_kl_divergence}
D_{\text{KL}}(\pi_{\theta}|| \pi_{\text{base}}) = \sum_x \pi_{\theta}(x)\log\left(\frac{\pi_{\theta}(x)}{\pi_{\text{base}}(x)}\right).
\end{equation}

Here, the term $D_{\text{KL}}(\pi_{\theta} || \pi_{\text{base}})$ represents the \textbf{Reverse KL Divergence}. It measures the information lost when the base policy $\pi_{\text{base}}$ approximates the new policy $\pi_{\theta}$. This choice heavily penalizes the new policy $\pi_{\theta}$ for exploring actions that the base policy $\pi_{\text{base}}$ considers to have a low probability. In other words, if the new policy $\pi_{\theta}$ assigns a high probability to an action where the base policy $\pi_{\text{base}}$ assigns a near-zero probability, this divergence will be very large. The effect is to strongly encourage the new policy $\pi_{\theta}$ to stick to the modes (high-probability regions) of the original base policy $\pi_{\text{base}}$, which restricts the exploration space.

\subsection{Dynamic Sampling Policy Optimization (DAPO) } 
DAPO is an enhancement of the GRPO algorithm, incorporating several key improvements. DAPO eliminates the KL penalty and refines the clipping mechanism, changing the upper bound from $(1+\varepsilon)$ to a fixed value of $(1+\varepsilon_{\text{upper}})$, where $\varepsilon_{\text{upper}}$ is set to 0.28. A core innovation of DAPO is its dynamic sampling mechanism, which moves beyond the ``all-or-nothing" approach to sampling. Additionally, the algorithm applies a token-level policy gradient loss and uses an overlong reward shaping technique.

\vspace{-3pt}
\section{Method for Divergence Definition}
\vspace{-3pt}
\label{definition}
For f-divergences like forward KL, the Divergence Definition and Generator implementations are equivalent. However, for divergences such as JS divergence, which require sampling from both the reference policy $\pi_{ref}$ and the new policy $\pi_{\theta}$:
\begin{equation}
\label{eq:js_divergence}
\text{JS}(\pi_{\text{base}}||\pi_{\theta}) = \frac{1}{2} D_{\text{KL}}\left(\pi_{\text{base}}||\frac{\pi_{\text{base}}+\pi_{\theta}}{2}\right) + \frac{1}{2} D_{\text{KL}}\left(\pi_{\theta}||\frac{\pi_{\text{base}}+\pi_{\theta}}{2}\right)
\end{equation}

Since the JS divergence is composed of two parts, we use data from the $\mathcal{D}_{\text{pef}}$ loss to compute the first part $D_{\text{KL}}\left(\pi_{\text{base}}||\frac{\pi_{\text{base}}+\pi_{\theta}}{2}\right)$. For the training on the $\mathcal{D}_{\text{exp}}$ loss, we introduce a reference model to calculate the value of the second part $D_{\text{KL}}\left(\pi_{\theta}||\frac{\pi_{\text{base}}+\pi_{\theta}}{2}\right)$.

The loss for DPH-JS on the $\mathcal{D}_{\text{exp}}$ can be written as:
\begin{equation}
\label{eq:dphrl_exp}
\begin{aligned}
\mathcal{L}_{\text{DPH-exp}}(\theta) = -\mathbb{E}_{\substack{(q,a)\sim \mathcal{D}_{\text{exp}} \\ o_i\sim \pi_{\theta_{\text{old}}}(\cdot\mid q)}} \Bigg[ & \frac{1}{G}\sum_{i=1}^{G} \frac{1}{|o_i|}\sum_{t=1}^{|o_i|} \min \left( \rho_{i,t}(\theta) \hat{A}_{i,t}, \ \text{clip} \left( \rho_{i,t}(\theta), 1 - \varepsilon, 1 + \varepsilon \right) \hat{A}_{i,t} \right) \\
& - \beta_{1} D_{\text{KL}}\left(\pi_{\theta}||\frac{\pi_{\text{base}}+\pi_{\theta}}{2}\right) \Bigg]
\end{aligned}
\end{equation}

By adjusting the value of $\beta_{1}$, we can make the values of the two JS components within the same batch as close as possible, thereby achieving balance during training. $\beta_{1}$ is not a hyperparameter; it only depends on the dataset size and $\eta$. Its formula is:
\begin{equation}
\beta_{1} = \eta * \frac{|\mathcal{D}_{\text{pef}}|}{|\mathcal{D}_{\text{exp}}|}
\end{equation}

The loss for DPH-JS on the $\mathcal{D}_{\text{pef}}$ can be expressed as:
\begin{equation}
\mathcal{L}_{\text{DPH-pef}} = \mathbb{E}_{q \sim \mathcal{D}_{\text{pef}}} \left[  D_{\text{KL}}\left(\pi_{\text{base}}||\frac{\pi_{\text{base}}+\pi_{\theta}}{2}\right) \right]
\end{equation}
\section{Training Details}
\label{training_details}
\subsection{Task Settings}
For the SQL task, we used the Llama-3.1-8B-Instruct and OmniSQL32B models. We performed RL training exclusively on the BIRD dataset~\cite{li2024bird} and then evaluated the models on both the BIRD and Spider (Cross-Domain)~\cite{yu2018spider} datasets. To validate the generalization of our method, we also tested it on a mathematical reasoning task, treating it as an out-of-domain (OOD) evaluation.

For the mathematical reasoning task, we used the Llama-3.1-8B-Instruct and Qwen2.5-Math-7B models. We conducted RL training on a DAPO-17K subset, having first filtered out problems that initially had a Pass@8 score of 0. For testing, we used seven different datasets: AIME24~\cite{AIME_AMC}, AIME25, AMC23~\cite{AIME_AMC}, Math500~\cite{MATH}, Olympiad~\cite{OlympiadBench}, Minerva~\cite{MinervaMath}, and College Math~\cite{College_math}.

\paragraph{Evaluation}
For the SQL tasks, the model will generate an SQL statement, and then we call the SQL executor to get the result. This result is typically a tuple, such as [(string, int), (string, int)]. We compute the Cartesian product of the prediction and the ground truth. For the math tasks, we use the official Qwen2.5 evaluation tool\footnote{\url{https://github.com/QwenLM/Qwen2.5-Math}} for a detailed assessment, as this tool provides tailored evaluations for each dataset.

\subsection{Method Settings}
\label{appendix:model_compare}
\begin{table}[ht]
    \centering
    \footnotesize % 缩小字体
    \caption{Configuration Comparison of RL Methods}
    \label{tab:method_comparison}
    \begin{tabular}{lcccccc} % 从 lccccc 改为 lcccccc
        \toprule
        \textbf{Method} & \textbf{Online Data} & \textbf{Offline Data} & \textbf{Offline Loss} & \textbf{$\eta$} & \textbf{$\epsilon_\text{high}$} & \textbf{Dynamic Sampling} \\
        \midrule
        GRPO & $\mathcal{D}_{\text{exp}}$,$\mathcal{D}_{\text{pef}}$ & - & - & - & 0.2 & $\times$ \\
        DAPO & $\mathcal{D}_{\text{exp}}$,$\mathcal{D}_{\text{pef}}$ & - & - & - & 0.28 & $\checkmark$ \\
        RKL & $\mathcal{D}_{\text{exp}}$,$\mathcal{D}_{\text{pef}}$ & - & - & 0.01 & 0.28 & $\checkmark$ \\
        $\alpha$-divergence & $\mathcal{D}_{\text{exp}}$ & $\mathcal{D}_{\text{pef}}$ & $\alpha$-divergence & 0.01 & 0.28 & $\checkmark$ \\
        DPH-F & $\mathcal{D}_{\text{exp}}$ & $\mathcal{D}_{\text{pef}}$ & Forward KL & 0.01 & 0.28 & $\checkmark$ \\
        DPH-JS & $\mathcal{D}_{\text{exp}}$ & $\mathcal{D}_{\text{pef}}$ & JS divergence & 0.2 & 0.28 & $\checkmark$ \\
        \bottomrule
    \end{tabular}
\end{table}

The implementation of RKL is consistent with the standard GRPO implementation. Our default implementation for DPH-RL uses the \textbf{Generator} method. The settings for our different methods are detailed in Table~\ref{tab:method_comparison}. For all RL algorithms, we consistently use a token-level loss. In most of our experiments, the $\epsilon_\text{high}$ for our DPH methods was set to 0.28. However, due to training instability in the Llama math experiments, we set the $\epsilon_\text{high}$ to 0.2 for all methods.

\subsection{Hyperparameters}
\label{append:parameters}
\begin{table}[ht]
    \centering
    \caption{Hyperparameters for RL Training}
    \label{parameter_table}
    \scriptsize
    \begin{tabular}{l|cccc}
        \toprule
        \textbf{Hyperparameter} & \textbf{Llama-SQL} & \textbf{OmniSQL-32B-SQL} & \textbf{Llama-Math} & \textbf{Qwen2.5-Math} \\
        \midrule
        Batch Size & 128 & 128 & 128 & 128 \\
        Learning Rate & 1e-6 & 1e-6 & 2e-7 & 2e-7 \\
        Rollout Temperature & 1.0 & 1.0 & 1.0 & 1.0 \\
        Rollout Top-p & 0.95 & 0.95 & 0.95 & 0.95 \\
        Validation Temperature & 1.0 & 1.0 & 0.6 & 0.6 \\
        Validation Top-p & 0.95 & 0.95 & 0.95 & 0.95 \\
        PPO Epochs & 1 & 1 & 1 & 1 \\
        Size of $\mathcal{D}_{\text{exp}}$ & 6710 & 5169 & 4599 & 7009 \\
        Size of $\mathcal{D}_{\text{pef}}$ & 2248 & 3789 & 1248 & 3199 \\
        Max Response Length & 4096 & 4096 & 2048 & 2048 \\
        Number of Rollouts & 16 & 16 & 16 & 16 \\
        Training Epochs & 4 & 4 & 8 & 8 \\
        \bottomrule
    \end{tabular}
\end{table}
All experiments were conducted on 32 NVIDIA A800-80G GPUs using the VeRL framework for our RL algorithm implementations. For all experiments, we set the number of rollouts to 16. In our setup, the training batch size was 128, and the batch size for the $D_{\text{pef}}$ data was 256. This resulted in using 2048 samples from $D_{\text{exp}}$ and 256 samples from $D_{\text{pef}}$ at each learning step, yielding an effective ratio of 8:1. A comprehensive list of the specific hyperparameters used is provided in Table~\ref{parameter_table}.

In the offline phase, we used the current model to perform 8 rollouts on the training data. For the SQL tasks, we separated the perfectly correct examples from the rollouts and used the remaining ones to construct $\mathcal{D}_{\text{exp}}$. For the DAPO-17k dataset used in the math tasks, we discarded all data where the model's responses were entirely incorrect.

\section{Additional Experiments}

\subsection{32B Experiments for SQL}
\begin{table*}[htpb]
\centering
\caption{Results on Omnisql-32B for SQL tasks. * Indicates that the results are from the Bird Bench.}

\label{omnisql-32b}
\scriptsize
\begin{adjustbox}{max width=\textwidth}
\begin{tabular}{lcccccc}
\toprule
\multirow{2}{*}{\textbf{Model}} & \multicolumn{3}{c}{\textbf{Bird}} & \multicolumn{3}{c}{\textbf{Spider}} \\
\cmidrule(lr){2-4} \cmidrule(lr){5-7}
& \textbf{Greedy} & \textbf{Pass@8} & \textbf{Pass@16} & \textbf{Greedy} & \textbf{Pass@8} & \textbf{Pass@16} \\
\midrule

\multicolumn{7}{c}{\textbf{OmniSQL-32B}} \\
\midrule
Base Model & 65.0 & 77.3 & 79.8 & 85.0 & 92.8 & 93.6 \\
\midrule
\quad \textit{Other Models} \\
Infly-RL-SQL-32B* & 70.1& - & - & - & - & - \\
XiYanSQL-QwenCoder-32B* & 67.0 &  - & - & - & - & - \\
Arctic-ExCoT-70B*&68.5 \\
Command A*& 63.5 \\ 
\midrule
\quad \textit{RL Models} \\ % Indented sub-header
GRPO & 69.4 & 76.3 & 78.8 & 84.7 & 91.6 & 92.4 \\
DAPO & 69.9 & 76.8 & 79.2 & 83.8 & 91.0 & 91.7 \\
DPH-F & 70.4 & 78.6 & 80.8 & \colorbox{mine}{86.1} & \colorbox{mine}{92.7} & \colorbox{mine}{93.3} \\
DPH-JS & \colorbox{mine}{70.5} & \colorbox{mine}{79.2} & \colorbox{mine}{81.9} & 84.9 & 91.9 & 93.1 \\
\bottomrule
\end{tabular}
\end{adjustbox}
\end{table*}
We also tested the effectiveness of DPH-F and DPH-JS on a larger 32B model and found them to be equally effective. As shown in Table~\ref{omnisql-32b}, please note that a greedy score of 70.5 on the Bird dataset is already approaching the performance limit of open-source Single-Models\footnote{\url{https://bird-bench.github.io/}}.

\subsection{Analysis of Training Progress}
\label{app:training_progress}
\begin{figure}[htbp]
    \centering
    % 第一个子图
    \begin{subfigure}[b]{0.48\textwidth}
        \centering
        \includegraphics[width=\textwidth]{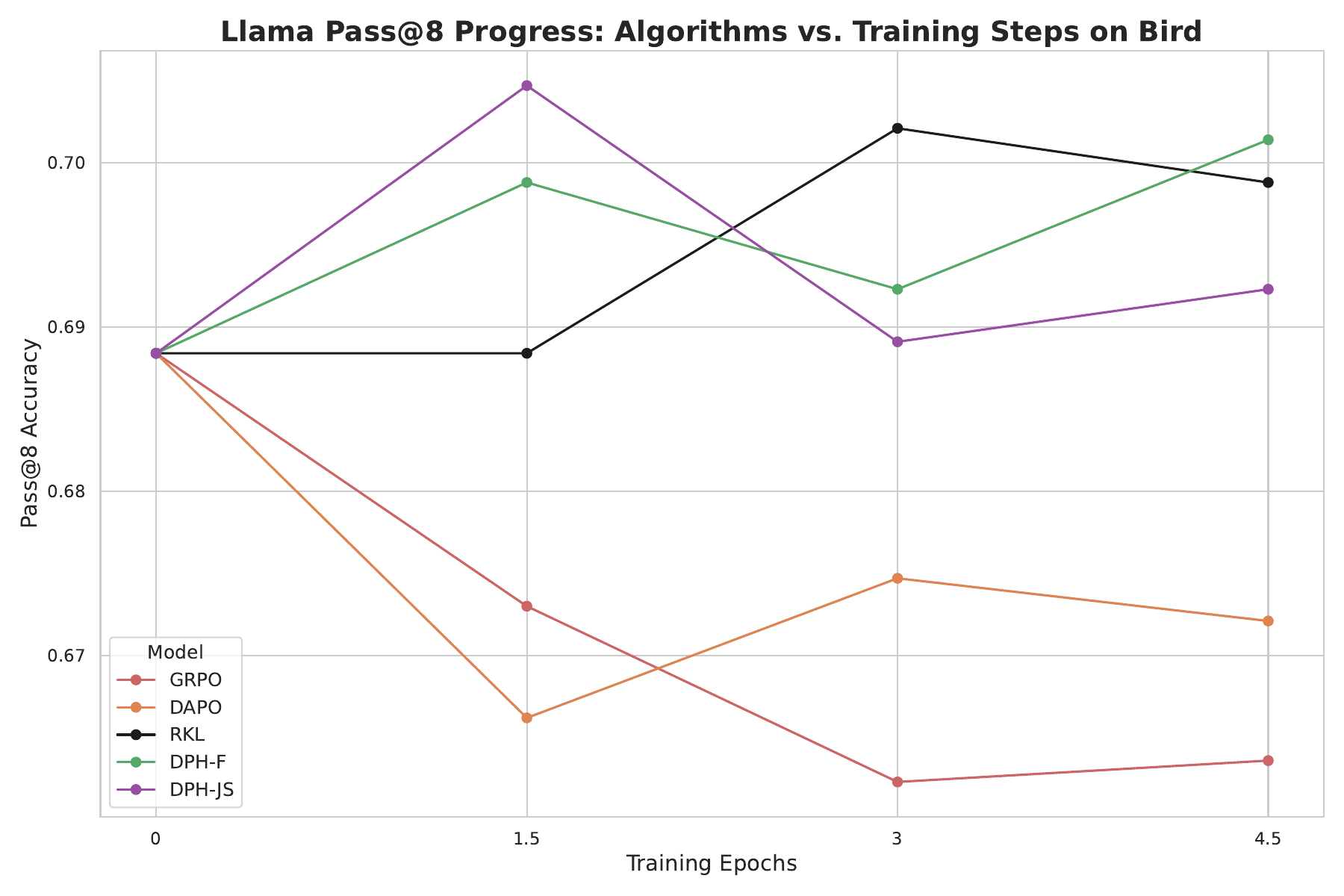}
        \caption{Llama Bird}
        \label{fig:sub-1}
    \end{subfigure}
    \hfill % 在两个子图之间添加一些空间
    % 第二个子图
    \begin{subfigure}[b]{0.48\textwidth}
        \centering
        \includegraphics[width=\textwidth]{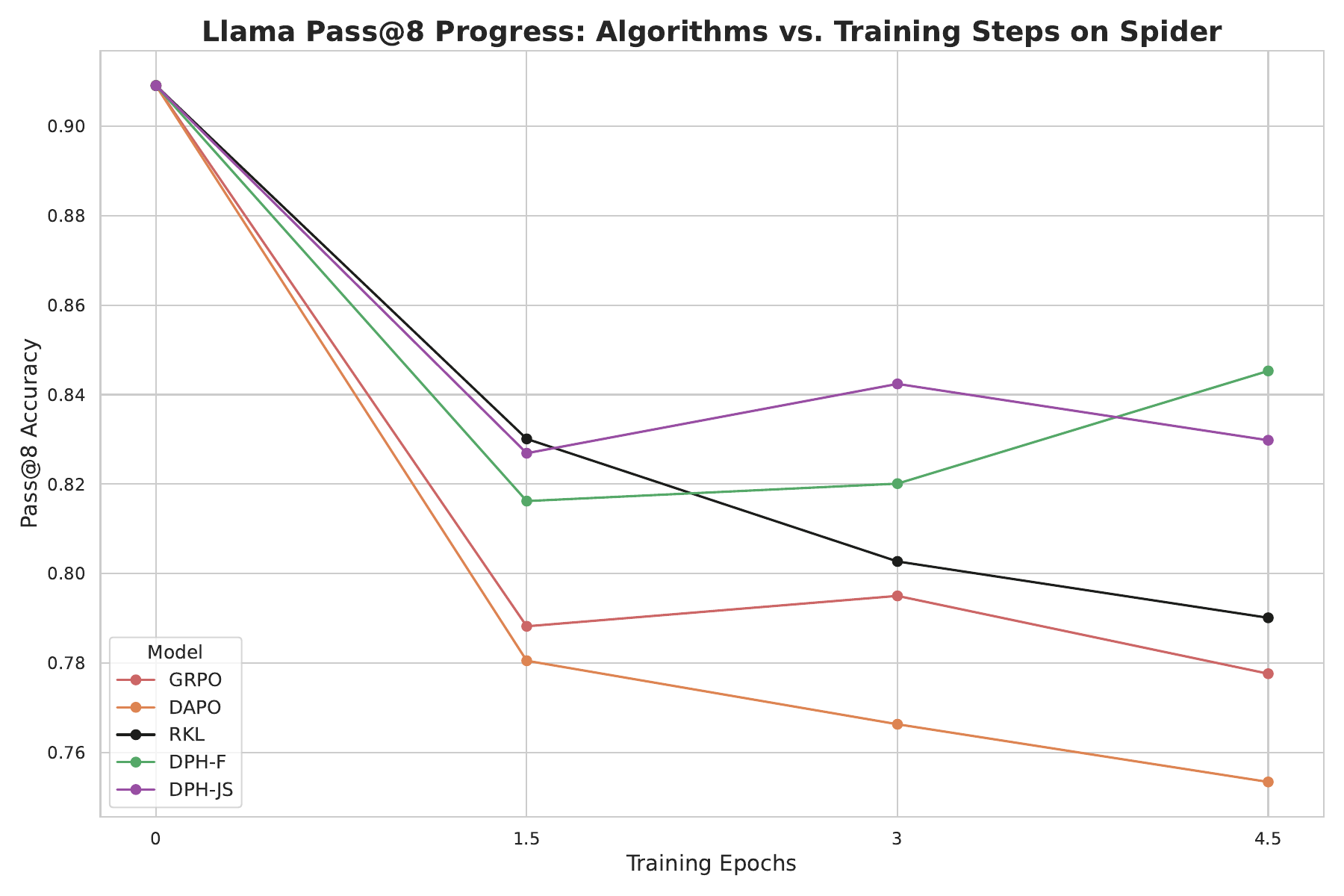}
        \caption{Llama Spider}
        \label{fig:sub-2}
    \end{subfigure}
    % 主标题，放在所有子图之后
    \caption{Llama Pass@8 progress: Algorithms vs. Training Steps}
    \label{figs:training_process}
\end{figure}

In Figure~\ref{figs:training_process}, we visualize the evolution of the Llama model's Pass@8 metric during training. On the Bird dataset, GRPO shows a clear, gradual collapse, while DAPO's performance continuously oscillates. Over the long term, DPH-F and DPH-JS both consistently maintain a higher Pass@8 value than the initial level. On the Spider dataset, DAPO performs extremely poorly, even falling below GRPO. Our experiments show that DAPO sacrifices out-of-domain generalization for improved performance on in-domain datasets. While RKL performs well on the Bird dataset, its performance also shows a gradual collapse on the Spider dataset as training progresses.

\subsection{Error Bar}

\begin{figure}[]

  \centering
  \includegraphics[width=0.6\textwidth]{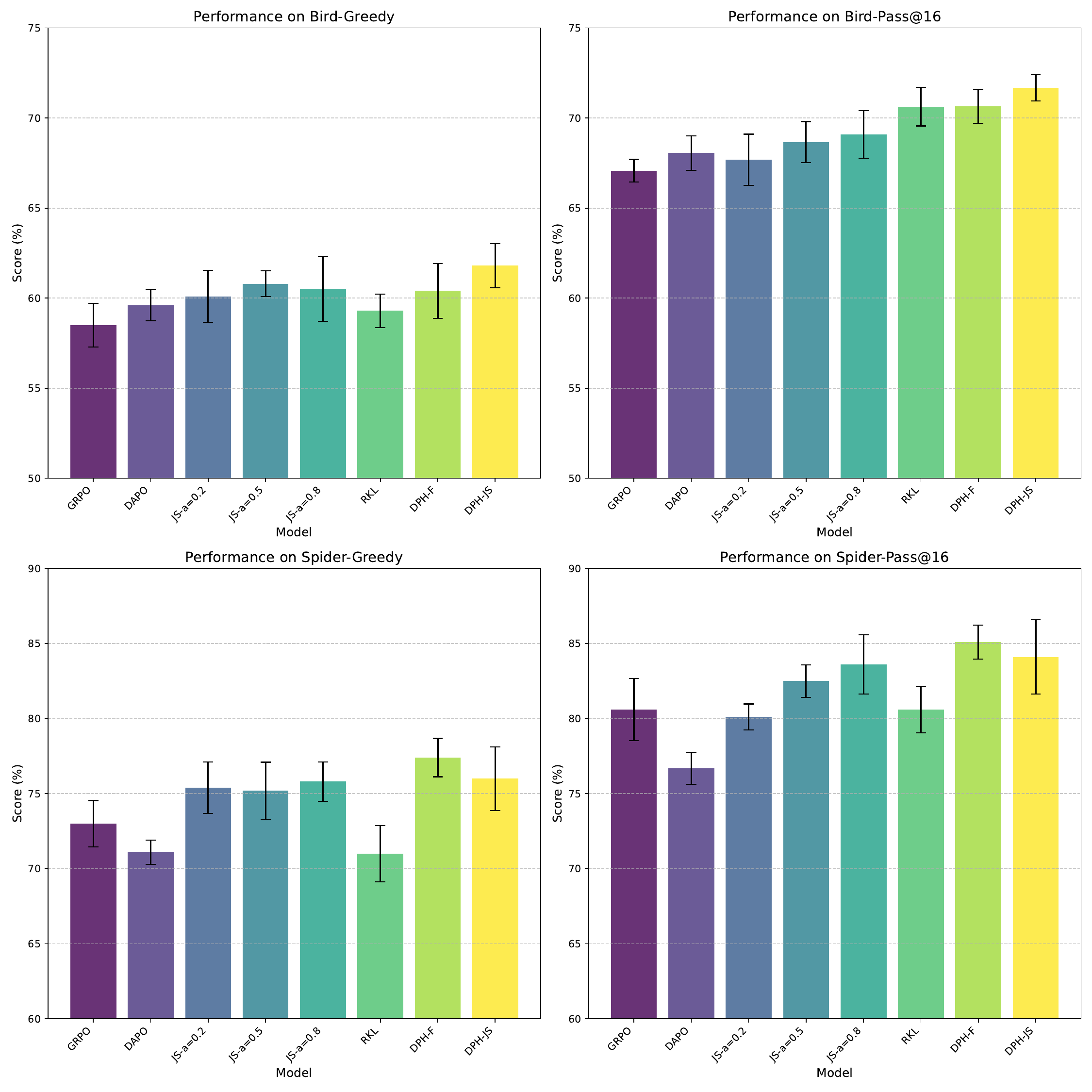}
  \caption{Error bar in Llama Sql. For each method in our Llama SQL experiments, we conducted three separate reinforcement learning training runs. We then selected the model that achieved the highest pass@16 score on the Bird.
  }
  \label{fig:error_bar}
\end{figure}

\subsection{Mean@k Metric for MATHEMATICAL REASONING TASK}

\begin{table*}[h]
\centering
\caption{The Mean@k metric for models trained on math datasets. To accelerate training, we filtered the DAPO-17k dataset for data with a Pass@8 score of 0. In the end, we trained Llama on 5.8k samples and Qwen on 10.2k samples.} 
\label{Table4}
\scriptsize
\adjustbox{max width=\linewidth}{
\renewcommand{\arraystretch}{1.1} % Increase the row height
\begin{tabular}{lcccccccc}
\toprule
\multirow{2}{*}{\textbf{Model}} & \multicolumn{3}{c}{\textbf{Mean@64}} & \multicolumn{3}{c}{\textbf{Mean@16}} & \textbf{Mean@8} & \multirow{2}{*}{\textbf{Avg}} \\
\cmidrule(lr){2-4} \cmidrule(lr){5-7} \cmidrule(lr){8-8}
& AIME24 & AIME25 & AMC23 & Math500 & Olympiad & Minerva & College Math \\
\midrule
\multicolumn{9}{c}{\textbf{Llama-3.1-8B-Instruct}} \\
\midrule
Base Model & 5.5 & 0.5 & 23.0 & 46.5 & 15.2 & 22.1 & 27.3 & 20.01 \\
RL Models \\
\quad-GRPO & 5.6 & \colorbox{mine}{0.6} & 23.3 & 47.7 & 14.9 & 25.4 & 29.1 & 20.94 \\
\quad-RKL & 6.4 & 0.2 & 22.3 & 45.6 & 13.0 & 24.8 & 28.6 & 20.13 \\
\quad-DPH-F & 6.2 & 0.5 & 23.8 & 48.1 & 15.5 & 25.4 & 29.7 & 21.31 \\
\quad-DPH-JS & \colorbox{mine}{8.4} & 0.5 & \colorbox{mine}{24.2} & \colorbox{mine}{48.4} & \colorbox{mine}{15.6} & \colorbox{mine}{27.4} & \colorbox{mine}{29.9} & \colorbox{mine}{22.06} \\
\midrule
\multicolumn{9}{c}{\textbf{Qwen2.5-Math-7B}} \\
\midrule
Base Model & 12.7 & 6.6 & 32.2 & 48.9 & 22.8 & 13.3 & 19.3 & 22.26 \\
RL Models \\
\quad-GRPO & 25.8 & 10.9 & 60.6 & 75.5 & 37.8 & 38.8 & 40.6 & 41.43 \\
\quad-DAPO & 25.9 & 11.1 & 61.3 & 76.3 & 37.6 & \colorbox{mine}{39.0} & 40.7 & 41.70 \\
\quad-RKL & 25.8 & 10.3 & 61.5 & 75.3 & 36.7 & 34.6 & 39.9 & 40.59 \\
\quad-DPH-F & 26.0 & 10.6 & 60.4 & 76.4 & 37.3 & 38.1 & 40.9 & 41.38 \\
\quad-DPH-JS & \colorbox{mine}{26.6} & \colorbox{mine}{11.2} & \colorbox{mine}{62.7} & \colorbox{mine}{76.5} & \colorbox{mine}{38.5} & 38.9 & \colorbox{mine}{41.9} & \colorbox{mine}{42.33} \\
\bottomrule
\end{tabular}
}
\end{table*}

To verify that our reinforcement learning training was effective on the math tasks, we evaluated the performance using the mean@k metric in Table~\ref{Table4}.

\section{Case Study}
We have listed the use cases and prompts for the two tasks in detail to facilitate future reproduction.

\subsection{SQL Task}
% \label{app:spideragent_framework}

% Following the ReAct \citep{yao2022react} prompting framework, which guides models in breaking down complex problems into manageable steps through an iterative chain of reasoning (Reasoning), action (Action), and observation (Observation) phases, where the current state is observed after each action before proceeding to the next step of reasoning, we build the Spider-Agent. Below, we will introduce the system messages and actions used within this framework.

\begin{tcolorbox}[title=\textbf{Case Study-SQL},colframe=purple!55!, breakable]
\colorbox{yellow!30}{System Prompt} \\
You are Qwen, created by Alibaba Cloud. You are a helpful assistant.\\
\colorbox{yellow!30}{User Prompt} \\
Task Overview:\\
You are a data science expert. Below, you are provided with a database schema and a natural language question. Your task is to understand the schema and generate a valid SQL query to answer the question.\\
Database Engine:\\
SQLite\\
Database Schema:\\
CREATE TABLE frpm (\\
CDSCode text, -- example: ['01100170109835', '01100170112607']\\
`Academic Year` text, -- example: ['2014-2015']\\
`County Code` text, -- example: ['01', '02']\\
`Free Meal Count (K-12)` real, -- example: [565.0, 186.0]\\
`Percent (\%) Eligible Free (K-12)` real, -- example: [0.519779208831647, 0.470886075949367]\\
 `2013-14 CALPADS Fall 1 Certification Status` integer, -- example: [1]\\
 PRIMARY KEY (CDSCode),\\
 CONSTRAINT fk\_frpm\_cdscode FOREIGN KEY (CDSCode) REFERENCES schools (CDSCode)\\
);\\
CREATE TABLE....
This schema describes the database's structure, including tables, columns, primary keys, foreign keys, and any relevant relationships or constraints.\\
Question:\\
Eligible free rate for K-12 = `Free Meal Count (K-12)` / `Enrollment (K-12)`.
What is the highest eligible free rate for K-12 students in the schools in Alameda County?\\
Instructions:\\
- Make sure you only output the information that is asked in the question. If the question asks for a specific column, make sure to only include that column in the SELECT clause, nothing more.\\
- The generated query should return all of the information asked in the question without any missing or extra information.\\
- Before generating the final SQL query, please think through the steps of how to write the query.\\

Output Format:\\
In your answer, please enclose the generated SQL query in a code block:\\
\begin{verbatim}
```sql
-- Your SQL query
```
\end{verbatim}

Take a deep breath and think step by step to find the correct SQL query.

\colorbox{yellow!30}{Response}\\
To find the codes of the schools with a total enrollment of over 500, we need to follow these steps:

1.  Identify the table that contains the enrollment data: The `frpm` table contains the `Enrollment (K-12)` and `Enrollment (Ages 5-17)` columns, which represent the enrollment data.
2.  Identify the relevant columns: We need to sum the `Enrollment (K-12)` and `Enrollment (Ages 5-17)` columns to get the total enrollment.
3.  Identify the filter condition: We need to filter the schools with a total enrollment of over 500.
4.  Write the SQL query: We can use a `SELECT` statement to retrieve the `CDSCode` column from the `frpm` table, filtering the schools where the sum of `Enrollment (K-12)` and `Enrollment (Ages 5-17)` is greater than 500.

Based on these steps, the SQL query to find the codes of the schools with a total enrollment of over 500 is:
\begin{verbatim}
```sql
SELECT frpm.CDSCode
FROM frpm
WHERE frpm.`Enrollment (K-12)` + frpm.
`Enrollment (Ages 5-17)` > 500;
```
\end{verbatim}
This query sums the `Enrollment (K-12)` and `Enrollment (Ages 5-17)` columns and filters the schools with a total enrollment of over 500, returning the `CDSCode` column.

\colorbox{yellow!30}{Execution results}
\begin{verbatim}
[['01100170109835'], ['01100170112607'], ['01100170124172']]
\end{verbatim}

\end{tcolorbox}

\subsection{Math Task}
\begin{tcolorbox}[title=\textbf{Case Study-Math},colframe=purple!55!, breakable]
\colorbox{yellow!30}{System Prompt} \\
Please reason step by step, and put your final answer within \textbackslash \textbackslash boxed\{\{\}\}.\\
\colorbox{yellow!30}{User Prompt} \\
Every morning Aya goes for a $9$-kilometer-long walk and stops at a coffee shop afterwards. When she walks at a constant speed of $s$ kilometers per hour, the walk takes her 4 hours, including $t$ minutes spent in the coffee shop. When she walks $s+2$ kilometers per hour, the walk takes her 2 hours and 24 minutes, including $t$ minutes spent in the coffee shop. Suppose Aya walks at $s+\frac{1}{2}$ kilometers per hour. Find the number of minutes the walk takes her, including the $t$ minutes spent in the coffee shop.\\
\colorbox{yellow!30}{Response}\\
To solve this problem, we need to first determine the value of \(s\) and \(t\). We know the following:

1. When Aya walks at \(s\) km/h, the total time including the coffee shop visit is 4 hours.
2. When Aya walks at \(s+2\) km/h, the total time including the coffee shop visit is 2 hours and 24 minutes (which is 2.4 hours).

First, we can set up the equations for the walking times:

1. For the first scenario:
\[ \frac{9}{s} + \frac{t}{60} = 4 \]

2. For the second scenario:
\[ \frac{9}{s+2} + \frac{t}{60} = 2.4 \]

We can solve these equations to find \(s\) and \(t\).

Let's subtract the second equation from the first to eliminate \(t\):

\[ \frac{9}{s} - \frac{9}{s+2} = 1.6 \]

This simplifies to:
\[ \frac{9(s+2) - 9s}{s(s+2)} = 1.6 \]
\[ \frac{18}{s(s+2)} = 1.6 \]
\[ 18 = 1.6s(s+2) \]
\[ 18 = 1.6s^2 + 3.2s \]
\[ 1.6s^2 + 3.2s - 18 = 0 \]

We can solve this quadratic equation using the quadratic formula \(s = \frac{-b \pm \sqrt{b^2 - 4ac}}{2a}\), where \(a = 1.6\), \(b = 3.2\), and \(c = -18\):

\[ s = \frac{-3.2 \pm \sqrt{3.2^2 - 4 \cdot 1.6 \cdot (-18)}}{2 \cdot 1.6} \]
\[ s = \frac{-3.2 \pm \sqrt{10.24 + 115.2}}{3.2} \]
\[ s = \frac{-3.2 \pm \sqrt{125.44}}{3.2} \]
\[ s = \frac{-3.2 \pm 11.2}{3.2} \]

This gives us two possible solutions:
\[ s = \frac{8}{3.2} = 2.5 \quad \text{(positive root)} \]
\[ s = \frac{-14.4}{3.2} = -4.5 \quad \text{(negative root, not physically meaningful)} \]

So, \(s = 2.5\) km/h. Now we can substitute \(s\) back into one of the original equations to find \(t\):

\[ \frac{9}{2.5} + \frac{t}{60} = 4 \]
\[ 3.6 + \frac{t}{60} = 4 \]
\[ \frac{t}{60} = 0.4 \]
\[ t = 24 \]

So, \(t = 24\) minutes. Now we need to find out how long the walk takes when Aya walks at \(s + \frac{1}{2} = 2.5 + 0.5 = 3\) km/h. The total time including the coffee shop visit will be:

\[ \frac{9}{3} + \frac{24}{60} = 3 + 0.4 = 3.4 \text{ hours} \]

Converting 3.4 hours to minutes, we get:

\[ 3.4 \times 60 = 204 \text{ minutes} \]

So the final answer is \( \boxed{204} \).

\end{tcolorbox}

\newpage
\section{Proof of Theorem 1}\label{app:proof}

The proof of Theorem 1 relies on bounding the total variation distance by various f-divergences. The following lemma, which is a direct consequence of the well-known Pinsker's inequality and its generalizations, establishes these relationships.

\begin{lemma}\label{lemma:pinsker_bounds}
For any two probability distributions $P$ and $Q$ defined on the same measurable space, the squared total variation distance, $D_{TV}^2(P, Q)$, is bounded by the Kullback-Leibler (KL), $\alpha$-, and Jensen-Shannon (JS) divergences as follows:
$$ D_{TV}^2(P, Q) \leq C \cdot D_f(P \| Q) $$
where $D_f$ is one of the specified divergences and $C$ is a positive constant that depends on the choice of $f$.
\end{lemma}

\begin{proof}
The proof consists of recalling the standard Pinsker-type inequalities \citep{fedotov2003refinements} for each specified f-divergence and arranging them into the desired form. We define the total variation distance as $D_{TV}(P, Q) = \frac{1}{2} \sum_x |P(x) - Q(x)|$.

\begin{itemize}
    \item \textbf{Kullback-Leibler (KL) Divergence:} The classical Pinsker's inequality states that
    $$ D_{KL}(P \| Q) \ge 2 D_{TV}^2(P, Q). $$
    Rearranging this directly yields the desired bound:
    $$ D_{TV}^2(P, Q) \le \frac{1}{2} D_{KL}(P \| Q). $$
    Here, $C = 1/2$.

    \item \textbf{$\alpha$-Divergence:} The relationship depends on the value of $\alpha$.
    \begin{itemize}
        \item For $\alpha > 1$, the $\alpha$-divergence is lower-bounded by the KL divergence. Consequently, the same constant applies:
        $$ D_{\alpha}(P \| Q) \ge D_{KL}(P \| Q) \ge 2 D_{TV}^2(P, Q) \implies D_{TV}^2(P, Q) \le \frac{1}{2} D_{\alpha}(P \| Q). $$
        \item For $\alpha \in (0, 1)$, a similar quadratic bound holds, of the form $D_{\alpha}(P \| Q) \ge C_\alpha D_{TV}^2(P, Q)$ for some constant $C_\alpha > 0$. For instance, the Hellinger distance, which corresponds to $\alpha=1/2$ (up to a factor), satisfies $H^2(P,Q) \ge 2 D_{TV}^2(P,Q)$.
    \end{itemize}
    In all cases, a suitable constant $C$ exists.

    \item \textbf{Jensen-Shannon (JS) Divergence:} For the JS divergence, there are two well-known lower bounds of this form.

 A commonly used quadratic inequality is:
    \begin{equation*}
        D_{\mathrm{JS}}(P \| Q) \ge \frac{1}{8} V(P, Q)^2
    \end{equation*}
 The tightest possible bound overall is non-quadratic:
    \begin{equation*}
         D_{\mathrm{JS}}(P \| Q) \ge \ln(2) - H_2\left(\frac{1+V(P,Q)}{2}\right)
    \end{equation*}
    where $H_2(p) = -p\ln(p) - (1-p)\ln(1-p)$. This provides the sharpest possible function for $g(V)$.

Thus, for each of the considered f-divergences, we have established the existence of a constant $C$ that satisfies the inequality. This concludes the proof.
    \end{itemize}
\end{proof}

Throughout the following proof,  we will repeatedly apply the lemma established in  \cite{schulman2015trust}.
\begin{lemma}[\cite{schulman2015trust}]\label{lemma:kakade}
    Given two policies $\pi$ and $\tilde{\pi}$, we have

    $$J(\tilde{\pi}) - J(\pi) = \mathbb{E}_{\tau \sim \tilde{\pi}} [\sum_{t=0}^{\infty}\gamma^tA_{\pi}(s_t,a_t)],$$
where the expectation is over the trajectories $\tau:=(s_0, a_0,s_1,a_1,...)$ and the notion $\mathbb{E}$ indicates that  $\tau$ are sampled from $\tilde{\pi}$ to generate $\tau$.
\end{lemma}

% \begin{lemma}[\cite{schulman2015trust}]
%     Given that $\pi$, $\tilde{\pi}$ are $\alpha$-coupled policies, define $A^{\tilde{\pi}| \pi}(s) = \mathbb{E}_{\tilde{a}\sim \tilde{\pi}} A_{\pi}(s,\tilde{a}) $,  for all $s$, we have 

% $$ |A^{\tilde{\pi}| \pi}(s)| \leq 2 \alpha \max_{s,a}|A_\pi(s,a)| $$
% \end{lemma}
Given that $\pi$, $\tilde{\pi}$ are $\alpha$-coupled policies, define $A^{\tilde{\pi}| \pi}(s) = \mathbb{E}_{\tilde{a}\sim \tilde{\pi}} A_{\pi}(s,\tilde{a}) $ we have following result. 
\begin{lemma}[Lemma 4 in \cite{kang2018policy}]\label{lemma:kang}
    let $(\pi, \tilde{\pi})$ be an $\alpha$ - coupled policy pair. Then 
    $$ |\mathbb{E}_{s_t \sim \tilde{\pi}}[A^{\tilde{\pi}|\pi}(s_t)| \leq 2\alpha(1-(1-\alpha)^t)\max_{s,a} |A_\pi(s,a)|   $$
\end{lemma}

Now we are ready to prove our theorem \ref{theorem:monotonic_improvement}.

\begin{proof}

Let's assume that the policy pairs $(\pi, \tilde{\pi})$ and $(\pi_{\text{pef}}, \tilde{\pi})$ are $\alpha$-coupled and $\beta$-coupled, respectively. This can be established if their total variation distances are bounded, i.e., $D_{TV}(\pi, \tilde{\pi}) \leq \alpha$ and $D_{TV}(\pi_{\text{pef}}, \tilde{\pi}) \leq \beta$. Our goal is to find a lower bound for the performance improvement $J(\tilde{\pi}) - L_{\pi}(\tilde{\pi})$. We proceed as follows:
    \begin{equation}
\begin{aligned}
        &J(\tilde{\pi}) - L_{\pi}(\tilde{\pi}) \\\overset{a}{=} & \mathbb{E}_{s_t \sim \tilde{\pi}}[\sum_{t=0}^{\infty} \gamma^t A^{\tilde{\pi}|\pi}(s_t)] - \mathbb{E}_{s_t \sim \pi} [\sum_{t=0}^{\infty} \gamma^t A^{\tilde{\pi}|\pi}(s_t)] \\
        = & \sum_{t=0}^{\infty} \big(\mathbb{E}_{s_t\sim\tilde{\pi}} [\gamma^t A^{\tilde{\pi}|\pi}(s_t)] - \mathbb{E}_{s_t\sim\pi_{pef}}[\gamma^t A^{\pi_{pef}|\pi}(s_t)] + \mathbb{E}_{s_t\sim\pi_{pef}} [\gamma^t A^{\pi_{pef}|\pi}(s_t)] - \mathbb{E}_{s_t\sim\pi}(\gamma^t A^{\tilde{\pi}|\pi}(s_t))\big) \\
        \overset{b}{=} & (\eta(\tilde{\pi}) - \eta(\pi)) - (\eta(\pi_{pef}) - \eta(\pi)) + \mathbb{E}_{s_t\sim\pi_{pef}} [\gamma^t A^{\pi_{pef}|\pi}(s_t)] - \mathbb{E}_{s_t\sim\pi}[\gamma^t A^{\tilde{\pi}|\pi}(s_t)]\\
         = & \eta(\tilde{\pi}) - \eta(\pi_{pef}) + \mathbb{E}_{s_t\sim\pi_{pef}} [\gamma^t A^{\pi_{pef}|\pi}(s_t)] - \mathbb{E}_{s_t\sim\pi}[\gamma^t A^{\tilde{\pi}|\pi}(s_t)]\\
         \overset{c}{=} & \sum_{t=0}^{\infty} \big( \mathbb{E}_{s_t\sim\tilde{\pi}} [\gamma^t A^{\tilde{\pi}|\pi_{pef}}(s_t)] +  \mathbb{E}_{s_t\sim\pi_{pef}} [\gamma^t A^{\pi_{pef}|\pi}(s_t)] - \mathbb{E}_{s_t\sim\pi}[\gamma^t A^{\tilde{\pi}|\pi}(s_t)]\big) \\
         \overset{d}{\geq} & \sum_{t=0}^{\infty} \big( -2\gamma^t \beta(1-(1-\beta)^t) \max_{(s,a)}|A_{\pi_{pef}}(s,a)|  + \gamma^t  \delta -  2\gamma^t \alpha(1-(1-\alpha)^t) \max_{(s,a)}|A_{\pi}(s,a)|  \big)\\
         = & -\frac{2\beta^2 \gamma \epsilon_p}{(1-\gamma)(1-\gamma(1-\beta))} - \frac{2\alpha^2 \gamma \epsilon_\pi}{(1-\gamma)(1-\gamma(1-\alpha))} + \frac{ \delta}{1-\gamma} \\
         \geq & -\frac{2\gamma(\beta^2\epsilon_p + \alpha^2 \epsilon_{\pi})}{(1-\gamma)^2} + \frac{\delta}{1-\gamma}
\end{aligned}
\end{equation}
where equality $a$, $b$, $c$ holds from Lemma \ref{lemma:kakade}; inequality $d$ applies Lemma \ref{lemma:kang}.

The final result is obtained by using the fact that $(1-\gamma(1-x)) > (1-\gamma)$ for $x \in (0,1)$, which simplifies the denominator.

The theorem's full statement is established by using Lemma \ref{lemma:pinsker_bounds} to connect the KL-divergence bounds to the total variation distances ($D_{TV}$), which in turn provides the conditions for the policy pairs to be $\alpha$-coupled and $\beta$-coupled.

\end{proof}

%% file: iclr2026_conference.bib
@misc{o1,
  title={Learning to reason with LLMs},
  author={OpenAI},
  year={2024},
  url = {https://openai.com/index/learning-to-reason-with-llms/}
}

@article{gpt4,
  title={{GPT4} technical report},
  author={OpenAI},
  journal={arXiv preprint arXiv:2303.08774},
  year={2023}
}

@article{schulman2017proximal,
  title={Proximal policy optimization algorithms},
  author={Schulman, John and Wolski, Filip and Dhariwal, Prafulla and Radford, Alec and Klimov, Oleg},
  journal={arXiv preprint arXiv:1707.06347},
  year={2017}
}

@article{deepseekmath,
  title={Deepseekmath: Pushing the limits of mathematical reasoning in open language models},
  author={Shao, Zhihong and Wang, Peiyi and Zhu, Qihao and Xu, Runxin and Song, Junxiao and Zhang, Mingchuan and Li, YK and Wu, Y and Guo, Daya},
  journal={arXiv preprint arXiv:2402.03300},
  year={2024}
}

@article{guo2025deepseek,
  title={Deepseek-r1: Incentivizing reasoning capability in llms via reinforcement learning},
  author={Guo, Daya and Yang, Dejian and Zhang, Haowei and Song, Junxiao and Zhang, Ruoyu and Xu, Runxin and Zhu, Qihao and Ma, Shirong and Wang, Peiyi and Bi, Xiao and others},
  journal={arXiv preprint arXiv:2501.12948},
  year={2025}
}

@misc{kl,
  title={Approximating kl divergence},
  author={John Schulman},
  year={2020},
  url = {http://joschu.net/blog/kl-approx.html}
}

@misc{dapo,
      title={DAPO: An Open-Source LLM Reinforcement Learning System at Scale}, 
      author={Qiying Yu and Zheng Zhang and Ruofei Zhu and Yufeng Yuan and Xiaochen Zuo and Yu Yue and Weinan Dai and Tiantian Fan and Gaohong Liu and Lingjun Liu and Xin Liu and Haibin Lin and Zhiqi Lin and Bole Ma and Guangming Sheng and Yuxuan Tong and Chi Zhang and Mofan Zhang and Wang Zhang and Hang Zhu and Jinhua Zhu and Jiaze Chen and Jiangjie Chen and Chengyi Wang and Hongli Yu and Yuxuan Song and Xiangpeng Wei and Hao Zhou and Jingjing Liu and Wei-Ying Ma and Ya-Qin Zhang and Lin Yan and Mu Qiao and Yonghui Wu and Mingxuan Wang},
      year={2025},
      eprint={2503.14476},
      archivePrefix={arXiv},
      primaryClass={cs.LG},
      url={https://arxiv.org/abs/2503.14476}, 
}

@misc{shen2025aienhanced,
  author = {Shen, Qiannan and Zhang, Jing},
  title = {AI-Enhanced Disaster Risk Prediction with Explainable SHAP Analysis: A Multi-Class Classification Approach Using XGBoost},
  year = {2025},
  month = {December},
  publisher = {Research Square},
  doi = {10.21203/rs.3.rs-8437180/v1},
  url = {https://www.researchsquare.com/article/rs-8437180/v1},
  note = {Preprint, Version 1, posted December 31, 2025}
}

@article{zhang2025count,
  title={Count Counts: Motivating Exploration in LLM Reasoning with Count-based Intrinsic Rewards},
  author={Zhang, Xuan and Li, Ruixiao and Zhou, Zhijian and Li, Long and Qin, Yulei and Li, Ke and Sun, Xing and Tan, Xiaoyu and Qu, Chao and Qi, Yuan},
  journal={arXiv preprint arXiv:2510.16614},
  year={2025}
}

@inproceedings{
yue2025does,
title={Does Reinforcement Learning Really Incentivize Reasoning Capacity in {LLM}s Beyond the Base Model?},
author={Yang Yue and Zhiqi Chen and Rui Lu and Andrew Zhao and Zhaokai Wang and Yang Yue and Shiji Song and Gao Huang},
booktitle={2nd AI for Math Workshop @ ICML 2025},
year={2025},
url={https://openreview.net/forum?id=upehLVgq1b}
}

@misc{he2025skyworkopenreasoner1,
      title={Skywork Open Reasoner 1 Technical Report}, 
      author={Jujie He and Jiacai Liu and Chris Yuhao Liu and Rui Yan and Chaojie Wang and Peng Cheng and Xiaoyu Zhang and Fuxiang Zhang and Jiacheng Xu and Wei Shen and Siyuan Li and Liang Zeng and Tianwen Wei and Cheng Cheng and Bo An and Yang Liu and Yahui Zhou},
      year={2025},
      eprint={2505.22312},
      archivePrefix={arXiv},
      primaryClass={cs.LG},
      url={https://arxiv.org/abs/2505.22312}, 
}

@misc{liu2025uftunifyingsupervisedreinforcement,
      title={UFT: Unifying Supervised and Reinforcement Fine-Tuning}, 
      author={Mingyang Liu and Gabriele Farina and Asuman Ozdaglar},
      year={2025},
      eprint={2505.16984},
      archivePrefix={arXiv},
      primaryClass={cs.LG},
      url={https://arxiv.org/abs/2505.16984}, 
}

@misc{liu2025ghpoadaptiveguidancestable,
      title={GHPO: Adaptive Guidance for Stable and Efficient LLM Reinforcement Learning}, 
      author={Ziru Liu and Cheng Gong and Xinyu Fu and Yaofang Liu and Ran Chen and Shoubo Hu and Suiyun Zhang and Rui Liu and Qingfu Zhang and Dandan Tu},
      year={2025},
      eprint={2507.10628},
      archivePrefix={arXiv},
      primaryClass={cs.LG},
      url={https://arxiv.org/abs/2507.10628}, 
}

@misc{liu2025acereasonnemotron11advancingmath,
      title={AceReason-Nemotron 1.1: Advancing Math and Code Reasoning through SFT and RL Synergy}, 
      author={Zihan Liu and Zhuolin Yang and Yang Chen and Chankyu Lee and Mohammad Shoeybi and Bryan Catanzaro and Wei Ping},
      year={2025},
      eprint={2506.13284},
      archivePrefix={arXiv},
      primaryClass={cs.CL},
      url={https://arxiv.org/abs/2506.13284}, 
}

@misc{Polaris2025,
    title = {POLARIS: A Post-Training Recipe for Scaling Reinforcement Learning on Advanced Reasoning Models},
    url = {https://hkunlp.github.io/blog/2025/Polaris},
    author = {An, Chenxin and Xie, Zhihui and Li, Xiaonan and Li, Lei and Zhang, Jun and Gong, Shansan and Zhong, Ming and Xu, Jingjing and Qiu, Xipeng and Wang, Mingxuan and Kong, Lingpeng},
    year = {2025}
}

@misc{liang2025pass1selfplayvariationalproblem,
      title={Beyond Pass@1: Self-Play with Variational Problem Synthesis Sustains RLVR}, 
      author={Xiao Liang and Zhongzhi Li and Yeyun Gong and Yelong Shen and Ying Nian Wu and Zhijiang Guo and Weizhu Chen},
      year={2025},
      eprint={2508.14029},
      archivePrefix={arXiv},
      primaryClass={cs.CL},
      url={https://arxiv.org/abs/2508.14029}, 
}

@misc{deng2025decomposingentropyperformanceexchangemissing,
      title={Decomposing the Entropy-Performance Exchange: The Missing Keys to Unlocking Effective Reinforcement Learning}, 
      author={Jia Deng and Jie Chen and Zhipeng Chen and Wayne Xin Zhao and Ji-Rong Wen},
      year={2025},
      eprint={2508.02260},
      archivePrefix={arXiv},
      primaryClass={cs.CL},
      url={https://arxiv.org/abs/2508.02260}, 
}

@misc{li2025curecriticaltokenguidedreconcatenationentropycollapse,
      title={CURE: Critical-Token-Guided Re-Concatenation for Entropy-Collapse Prevention}, 
      author={Qingbin Li and Rongkun Xue and Jie Wang and Ming Zhou and Zhi Li and Xiaofeng Ji and Yongqi Wang and Miao Liu and Zheming Yang and Minghui Qiu and Jing Yang},
      year={2025},
      eprint={2508.11016},
      archivePrefix={arXiv},
      primaryClass={cs.LG},
      url={https://arxiv.org/abs/2508.11016}, 
}

@misc{cheng2025reasoningexplorationentropyperspective,
      title={Reasoning with Exploration: An Entropy Perspective on Reinforcement Learning for LLMs}, 
      author={Daixuan Cheng and Shaohan Huang and Xuekai Zhu and Bo Dai and Wayne Xin Zhao and Zhenliang Zhang and Furu Wei},
      year={2025},
      eprint={2506.14758},
      archivePrefix={arXiv},
      primaryClass={cs.CL},
      url={https://arxiv.org/abs/2506.14758}, 
}

@misc{wang20258020rulehighentropyminority,
      title={Beyond the 80/20 Rule: High-Entropy Minority Tokens Drive Effective Reinforcement Learning for LLM Reasoning}, 
      author={Shenzhi Wang and Le Yu and Chang Gao and Chujie Zheng and Shixuan Liu and Rui Lu and Kai Dang and Xionghui Chen and Jianxin Yang and Zhenru Zhang and Yuqiong Liu and An Yang and Andrew Zhao and Yang Yue and Shiji Song and Bowen Yu and Gao Huang and Junyang Lin},
      year={2025},
      eprint={2506.01939},
      archivePrefix={arXiv},
      primaryClass={cs.CL},
      url={https://arxiv.org/abs/2506.01939}, 
}

@misc{cui2025entropymechanismreinforcementlearning,
      title={The Entropy Mechanism of Reinforcement Learning for Reasoning Language Models}, 
      author={Ganqu Cui and Yuchen Zhang and Jiacheng Chen and Lifan Yuan and Zhi Wang and Yuxin Zuo and Haozhan Li and Yuchen Fan and Huayu Chen and Weize Chen and Zhiyuan Liu and Hao Peng and Lei Bai and Wanli Ouyang and Yu Cheng and Bowen Zhou and Ning Ding},
      year={2025},
      eprint={2505.22617},
      archivePrefix={arXiv},
      primaryClass={cs.LG},
      url={https://arxiv.org/abs/2505.22617}, 
}

@inproceedings{htl,
    title = "How Do Humans Write Code? Large Models Do It the Same Way Too",
    author = "Li, Long  and
      He, Xuzheng  and
      Wang, Haozhe  and
      Wang, Linlin  and
      He, Liang",
    editor = "Al-Onaizan, Yaser  and
      Bansal, Mohit  and
      Chen, Yun-Nung",
    booktitle = "Proceedings of the 2024 Conference on Empirical Methods in Natural Language Processing",
    month = nov,
    year = "2024",
    address = "Miami, Florida, USA",
    publisher = "Association for Computational Linguistics",
    url = "https://aclanthology.org/2024.emnlp-main.267/",
    doi = "10.18653/v1/2024.emnlp-main.267",
    pages = "4638--4649"
}

@misc{qwen3,
      title={Qwen3 Technical Report}, 
      author={An Yang and Anfeng Li and Baosong Yang and Beichen Zhang and Binyuan Hui and Bo Zheng and Bowen Yu and Chang Gao and Chengen Huang and Chenxu Lv and Chujie Zheng and Dayiheng Liu and Fan Zhou and Fei Huang and Feng Hu and Hao Ge and Haoran Wei and Huan Lin and Jialong Tang and Jian Yang and Jianhong Tu and Jianwei Zhang and Jianxin Yang and Jiaxi Yang and Jing Zhou and Jingren Zhou and Junyang Lin and Kai Dang and Keqin Bao and Kexin Yang and Le Yu and Lianghao Deng and Mei Li and Mingfeng Xue and Mingze Li and Pei Zhang and Peng Wang and Qin Zhu and Rui Men and Ruize Gao and Shixuan Liu and Shuang Luo and Tianhao Li and Tianyi Tang and Wenbiao Yin and Xingzhang Ren and Xinyu Wang and Xinyu Zhang and Xuancheng Ren and Yang Fan and Yang Su and Yichang Zhang and Yinger Zhang and Yu Wan and Yuqiong Liu and Zekun Wang and Zeyu Cui and Zhenru Zhang and Zhipeng Zhou and Zihan Qiu},
      year={2025},
      eprint={2505.09388},
      archivePrefix={arXiv},
      primaryClass={cs.CL},
      url={https://arxiv.org/abs/2505.09388}, 
}

@misc{ki1.5,
      title={Kimi k1.5: Scaling Reinforcement Learning with LLMs}, 
      author={Kimi Team and Angang Du and Bofei Gao and Bowei Xing and Changjiu Jiang and Cheng Chen and Cheng Li and Chenjun Xiao and Chenzhuang Du and Chonghua Liao and Chuning Tang and Congcong Wang and Dehao Zhang and Enming Yuan and Enzhe Lu and Fengxiang Tang and Flood Sung and Guangda Wei and Guokun Lai and Haiqing Guo and Han Zhu and Hao Ding and Hao Hu and Hao Yang and Hao Zhang and Haotian Yao and Haotian Zhao and Haoyu Lu and Haoze Li and Haozhen Yu and Hongcheng Gao and Huabin Zheng and Huan Yuan and Jia Chen and Jianhang Guo and Jianlin Su and Jianzhou Wang and Jie Zhao and Jin Zhang and Jingyuan Liu and Junjie Yan and Junyan Wu and Lidong Shi and Ling Ye and Longhui Yu and Mengnan Dong and Neo Zhang and Ningchen Ma and Qiwei Pan and Qucheng Gong and Shaowei Liu and Shengling Ma and Shupeng Wei and Sihan Cao and Siying Huang and Tao Jiang and Weihao Gao and Weimin Xiong and Weiran He and Weixiao Huang and Weixin Xu and Wenhao Wu and Wenyang He and Xianghui Wei and Xianqing Jia and Xingzhe Wu and Xinran Xu and Xinxing Zu and Xinyu Zhou and Xuehai Pan and Y. Charles and Yang Li and Yangyang Hu and Yangyang Liu and Yanru Chen and Yejie Wang and Yibo Liu and Yidao Qin and Yifeng Liu and Ying Yang and Yiping Bao and Yulun Du and Yuxin Wu and Yuzhi Wang and Zaida Zhou and Zhaoji Wang and Zhaowei Li and Zhen Zhu and Zheng Zhang and Zhexu Wang and Zhilin Yang and Zhiqi Huang and Zihao Huang and Ziyao Xu and Zonghan Yang and Zongyu Lin},
      year={2025},
      eprint={2501.12599},
      archivePrefix={arXiv},
      primaryClass={cs.AI},
      url={https://arxiv.org/abs/2501.12599}, 
}

@article{ouyang2022training,
  title={Training language models to follow instructions with human feedback},
  author={Ouyang, Long and Wu, Jeffrey and Jiang, Xu and Almeida, Diogo and Wainwright, Carroll and Mishkin, Pamela and Zhang, Chong and Agarwal, Sandhini and Slama, Katarina and Ray, Alex and others},
  journal={Advances in neural information processing systems},
  volume={35},
  pages={27730--27744},
  year={2022}
}

@article{zhang2025rlep,
  title={Rlep: Reinforcement learning with experience replay for llm reasoning},
  author={Zhang, Hongzhi and Fu, Jia and Zhang, Jingyuan and Fu, Kai and Wang, Qi and Zhang, Fuzheng and Zhou, Guorui},
  journal={arXiv preprint arXiv:2507.07451},
  year={2025}
}

@inproceedings{
wang2024beyond,
title={Beyond Reverse {KL}: Generalizing Direct Preference Optimization with Diverse Divergence Constraints},
author={Chaoqi Wang and Yibo Jiang and Chenghao Yang and Han Liu and Yuxin Chen},
booktitle={The Twelfth International Conference on Learning Representations},
year={2024},
url={https://openreview.net/forum?id=2cRzmWXK9N}
}

@article{walder2025pass,
  title={Pass@ K Policy Optimization: Solving Harder Reinforcement Learning Problems},
  author={Walder, Christian and Karkhanis, Deep},
  journal={arXiv preprint arXiv:2505.15201},
  year={2025}
}

@article{rl_plus,
  title={RL-PLUS: Countering Capability Boundary Collapse of LLMs in Reinforcement Learning with Hybrid-policy Optimization},
  author={Dong, Yihong and Jiang, Xue and Tao, Yongding and Liu, Huanyu and Zhang, Kechi and Mou, Lili and Cao, Rongyu and Ma, Yingwei and Chen, Jue and Li, Binhua and others},
  journal={arXiv preprint arXiv:2508.00222},
  year={2025}
}

@misc{lambert2025tulu3pushingfrontiers,
      title={Tulu 3: Pushing Frontiers in Open Language Model Post-Training}, 
      author={Nathan Lambert and Jacob Morrison and Valentina Pyatkin and Shengyi Huang and Hamish Ivison and Faeze Brahman and Lester James V. Miranda and Alisa Liu and Nouha Dziri and Shane Lyu and Yuling Gu and Saumya Malik and Victoria Graf and Jena D. Hwang and Jiangjiang Yang and Ronan Le Bras and Oyvind Tafjord and Chris Wilhelm and Luca Soldaini and Noah A. Smith and Yizhong Wang and Pradeep Dasigi and Hannaneh Hajishirzi},
      year={2025},
      eprint={2411.15124},
      archivePrefix={arXiv},
      primaryClass={cs.CL},
      url={https://arxiv.org/abs/2411.15124}, 
}

@article{kazemnejad2024vineppo,
  title={Vineppo: Unlocking rl potential for llm reasoning through refined credit assignment},
  author={Kazemnejad, Amirhossein and Aghajohari, Milad and Portelance, Eva and Sordoni, Alessandro and Reddy, Siva and Courville, Aaron and Roux, Nicolas Le},
  journal={arXiv preprint arXiv:2410.01679},
  year={2024}
}

@inproceedings{yu2018spider,
  title={Spider: A Large-Scale Human-Labeled Dataset for Complex and Cross-Domain Semantic Parsing and Text-to-SQL Task},
  author={Yu, Tao and Zhang, Rui and Yang, Kai and Yasunaga, Michihiro and Wang, Dongxu and Li, Zifan and Ma, James and Li, Irene and Yao, Qingning and Roman, Shanelle and others},
  booktitle={Proceedings of the 2018 Conference on Empirical Methods in Natural Language Processing},
  pages={3911--3921},
  year={2018}
}

@article{li2024bird,
  title={Can llm already serve as a database interface? a big bench for large-scale database grounded text-to-sqls},
  author={Li, Jinyang and Hui, Binyuan and Qu, Ge and Yang, Jiaxi and Li, Binhua and Li, Bowen and Wang, Bailin and Qin, Bowen and Geng, Ruiying and Huo, Nan and others},
  journal={Advances in Neural Information Processing Systems},
  volume={36},
  year={2024}
}

@misc{meta2024llama3,
  title = {Introducing Meta {Llama 3}: The most capable openly available {LLM} to date},
  author = {{Meta AI}},
  year = {2024},
  month = apr,
  url = {https://ai.meta.com/blog/meta-llama-3/},
  note = {Accessed: 2024-04-18}
}

@misc{qwen2.5,
    title = {Qwen2.5: A Party of Foundation Models},
    url = {https://qwenlm.github.io/blog/qwen2.5/},
    author = {Qwen Team},
    month = {September},
    year = {2024}
}

@inproceedings{MATH,
  author       = {Dan Hendrycks and
                  Collin Burns and
                  Saurav Kadavath and
                  Akul Arora and
                  Steven Basart and
                  Eric Tang and
                  Dawn Song and
                  Jacob Steinhardt},
  title        = {Measuring Mathematical Problem Solving With the {MATH} Dataset},
  booktitle    = {NeurIPS Datasets and Benchmarks},
  year         = {2021}
}

@inproceedings{MinervaMath,
  author       = {Aitor Lewkowycz and
                  Anders Andreassen and
                  David Dohan and
                  Ethan Dyer and
                  Henryk Michalewski and
                  Vinay V. Ramasesh and
                  Ambrose Slone and
                  Cem Anil and
                  Imanol Schlag and
                  Theo Gutman{-}Solo and
                  Yuhuai Wu and
                  Behnam Neyshabur and
                  Guy Gur{-}Ari and
                  Vedant Misra},
  title        = {Solving Quantitative Reasoning Problems with Language Models},
  booktitle    = {NeurIPS},
  year         = {2022}
}

@inproceedings{College_math,
author = {Tang, Zhengyang and Zhang, Xingxing and Wang, Benyou and Wei, Furu},
title = {MathScale: scaling instruction tuning for mathematical reasoning},
year = {2024},
publisher = {JMLR.org},
booktitle = {Proceedings of the 41st International Conference on Machine Learning},
articleno = {1954},
numpages = {16},
location = {Vienna, Austria},
series = {ICML'24}
}

@inproceedings{OlympiadBench,
  author       = {Chaoqun He and
                  Renjie Luo and
                  Yuzhuo Bai and
                  Shengding Hu and
                  Zhen Leng Thai and
                  Junhao Shen and
                  Jinyi Hu and
                  Xu Han and
                  Yujie Huang and
                  Yuxiang Zhang and
                  Jie Liu and
                  Lei Qi and
                  Zhiyuan Liu and
                  Maosong Sun},
  title        = {OlympiadBench: {A} Challenging Benchmark for Promoting {AGI} with
                  Olympiad-Level Bilingual Multimodal Scientific Problems},
  booktitle    = {{ACL} {(1)}},
  pages        = {3828--3850},
  publisher    = {Association for Computational Linguistics},
  year         = {2024}
}

@misc{AIME_AMC,
  author       = {Jia Li and Edward Beeching and Lewis Tunstall and Ben Lipkin and Roman Soletskyi and Shengyi Huang and Kashif Rasul and Longhui Yu and Albert Q. Jiang and Ziju Shen and others},
  title        = {Numinamath: The largest public dataset in AI4Maths with 860k pairs of competition math problems and solutions},
  year         = {2024},
  howpublished = {\url{https://huggingface.co/datasets/Numinamath}},
  note         = {Hugging Face repository, 13:9}
}

@misc{agarwal2025unreasonableeffectivenessentropyminimization,
      title={The Unreasonable Effectiveness of Entropy Minimization in LLM Reasoning}, 
      author={Shivam Agarwal and Zimin Zhang and Lifan Yuan and Jiawei Han and Hao Peng},
      year={2025},
      eprint={2505.15134},
      archivePrefix={arXiv},
      primaryClass={cs.LG},
      url={https://arxiv.org/abs/2505.15134}, 
}

@misc{chen2025seedgrposemanticentropyenhanced,
      title={SEED-GRPO: Semantic Entropy Enhanced GRPO for Uncertainty-Aware Policy Optimization}, 
      author={Minghan Chen and Guikun Chen and Wenguan Wang and Yi Yang},
      year={2025},
      eprint={2505.12346},
      archivePrefix={arXiv},
      primaryClass={cs.AI},
      url={https://arxiv.org/abs/2505.12346}, 
}

@misc{minimax2025minimaxm1scalingtesttimecompute,
      title={MiniMax-M1: Scaling Test-Time Compute Efficiently with Lightning Attention}, 
      author={MiniMax and : and Aili Chen and Aonian Li and Bangwei Gong and Binyang Jiang and Bo Fei and Bo Yang and Boji Shan and Changqing Yu and Chao Wang and Cheng Zhu and Chengjun Xiao and Chengyu Du and Chi Zhang and Chu Qiao and Chunhao Zhang and Chunhui Du and Congchao Guo and Da Chen and Deming Ding and Dianjun Sun and Dong Li and Enwei Jiao and Haigang Zhou and Haimo Zhang and Han Ding and Haohai Sun and Haoyu Feng and Huaiguang Cai and Haichao Zhu and Jian Sun and Jiaqi Zhuang and Jiaren Cai and Jiayuan Song and Jin Zhu and Jingyang Li and Jinhao Tian and Jinli Liu and Junhao Xu and Junjie Yan and Junteng Liu and Junxian He and Kaiyi Feng and Ke Yang and Kecheng Xiao and Le Han and Leyang Wang and Lianfei Yu and Liheng Feng and Lin Li and Lin Zheng and Linge Du and Lingyu Yang and Lunbin Zeng and Minghui Yu and Mingliang Tao and Mingyuan Chi and Mozhi Zhang and Mujie Lin and Nan Hu and Nongyu Di and Peng Gao and Pengfei Li and Pengyu Zhao and Qibing Ren and Qidi Xu and Qile Li and Qin Wang and Rong Tian and Ruitao Leng and Shaoxiang Chen and Shaoyu Chen and Shengmin Shi and Shitong Weng and Shuchang Guan and Shuqi Yu and Sichen Li and Songquan Zhu and Tengfei Li and Tianchi Cai and Tianrun Liang and Weiyu Cheng and Weize Kong and Wenkai Li and Xiancai Chen and Xiangjun Song and Xiao Luo and Xiao Su and Xiaobo Li and Xiaodong Han and Xinzhu Hou and Xuan Lu and Xun Zou and Xuyang Shen and Yan Gong and Yan Ma and Yang Wang and Yiqi Shi and Yiran Zhong and Yonghong Duan and Yongxiang Fu and Yongyi Hu and Yu Gao and Yuanxiang Fan and Yufeng Yang and Yuhao Li and Yulin Hu and Yunan Huang and Yunji Li and Yunzhi Xu and Yuxin Mao and Yuxuan Shi and Yuze Wenren and Zehan Li and Zelin Li and Zhanxu Tian and Zhengmao Zhu and Zhenhua Fan and Zhenzhen Wu and Zhichao Xu and Zhihang Yu and Zhiheng Lyu and Zhuo Jiang and Zibo Gao and Zijia Wu and Zijian Song and Zijun Sun},
      year={2025},
      eprint={2506.13585},
      archivePrefix={arXiv},
      primaryClass={cs.CL},
      url={https://arxiv.org/abs/2506.13585}, 
}

@misc{prabhudesai2025maximizingconfidenceimprovesreasoning,
      title={Maximizing Confidence Alone Improves Reasoning}, 
      author={Mihir Prabhudesai and Lili Chen and Alex Ippoliti and Katerina Fragkiadaki and Hao Liu and Deepak Pathak},
      year={2025},
      eprint={2505.22660},
      archivePrefix={arXiv},
      primaryClass={cs.LG},
      url={https://arxiv.org/abs/2505.22660}, 
}

@inproceedings{mahdavi2025beyond,
  title={Beyond Accuracy: A Policy Gradient Reweighting Approach for Pass@ K Maximization in LLMs},
  author={Mahdavi, Sadegh and Li, Muchen and Liu, Kaiwen and Liao, Renjie and Thrampoulidis, Christos},
  booktitle={2nd AI for Math Workshop@ ICML 2025}
}

@book{bishop2006pattern,
  title={Pattern recognition and machine learning},
  author={Bishop, Christopher M and Nasrabadi, Nasser M},
  volume={4},
  number={4},
  year={2006},
  publisher={Springer}
}

@inproceedings{hamadanian2023online,
  title={Online Reinforcement Learning in Non-Stationary Context-Driven Environments},
  author={Hamadanian, Pouya and Nasr-Esfahany, Arash and Schwarzkopf, Malte and Sen, Siddartha and Alizadeh, Mohammad},
 booktitle={International Conference on Learning Representations},
  year={2025}
}

@article{dong2025rl,
  title={RL-PLUS: Countering Capability Boundary Collapse of LLMs in Reinforcement Learning with Hybrid-policy Optimization},
  author={Dong, Yihong and Jiang, Xue and Tao, Yongding and Liu, Huanyu and Zhang, Kechi and Mou, Lili and Cao, Rongyu and Ma, Yingwei and Chen, Jue and Li, Binhua and others},
  journal={arXiv preprint arXiv:2508.00222},
  year={2025}
}

@article{yan2504learning,
  title={Learning to reason under off-policy guidance, 2025},
  author={Yan, Jianhao and Li, Yafu and Hu, Zican and Wang, Zhi and Cui, Ganqu and Qu, Xiaoye and Cheng, Yu and Zhang, Yue},
  journal={URL https://arxiv. org/abs/2504.14945}
}

@inproceedings{wangbeyond,
  title={Beyond Reverse KL: Generalizing Direct Preference Optimization with Diverse Divergence Constraints},
  author={Wang, Chaoqi and Jiang, Yibo and Yang, Chenghao and Liu, Han and Chen, Yuxin},
  booktitle={The Twelfth International Conference on Learning Representations}
}

@inproceedings{han2025f,
  title={$ f $-PO: Generalizing Preference Optimization with $ f $-divergence Minimization},
  author={Han, Jiaqi and Jiang, Mingjian and Song, Yuxuan and Ermon, Stefano and Xu, Minkai},
  booktitle={International Conference on Artificial Intelligence and Statistics},
  pages={1144--1152},
  year={2025},
  organization={PMLR}
}

@inproceedings{schulman2015trust,
  title={Trust region policy optimization},
  author={Schulman, John and Levine, Sergey and Abbeel, Pieter and Jordan, Michael and Moritz, Philipp},
  booktitle={International conference on machine learning},
  pages={1889--1897},
  year={2015},
  organization={PMLR}
}

@inproceedings{kang2018policy,
  title={Policy optimization with demonstrations},
  author={Kang, Bingyi and Jie, Zequn and Feng, Jiashi},
  booktitle={International conference on machine learning},
  pages={2469--2478},
  year={2018},
  organization={PMLR}
}

@inproceedings{kakade2002approximately,
  title={Approximately optimal approximate reinforcement learning},
  author={Kakade, Sham and Langford, John},
  booktitle={Proceedings of the nineteenth international conference on machine learning},
  pages={267--274},
  year={2002}
}

@article{fedotov2003refinements,
  title={Refinements of Pinsker's inequality},
  author={Fedotov, Alexei A and Harremo{\"e}s, Peter and Topsoe, Flemming},
  journal={IEEE Transactions on Information Theory},
  volume={49},
  number={6},
  pages={1491--1498},
  year={2003},
  publisher={IEEE}
}
